\documentclass[journal]{IEEEtran}
\ifCLASSINFOpdf
\else
\fi
\usepackage{color,graphicx,amsthm,amssymb,amsmath,cite,xspace,setspace}

\usepackage{epsfig}
\usepackage{calc}
\usepackage{multicol}
\usepackage{pslatex}
\usepackage[small]{caption}
\usepackage{multirow}
\usepackage[ruled,linesnumbered]{algorithm2e}
\usepackage{algpseudocode}
\usepackage{tikz}
\usepackage{pgfplots}
\usetikzlibrary{spy}
\tikzset{annot/.style={draw=black,fill=white,text=black}}
\usepackage[caption=false]{subfig}

\newcommand{\bbeta}{\mbox{\boldmath  $\beta$}}

\newcommand{\bSigma}{\mbox{\boldmath $\Sigma$}}

\newcommand{\bmu}{\mbox{\boldmath   $\mu$}}

\newcommand{\Ymat}{{\bf Y}}
\newcommand{\Zmat}{{\bf Z}}
\newcommand{\Bmat}{{\bf B}}
\newcommand{\Mmat}{{\bf M}}
\newcommand{\Cmat}{{\bf C}}
\newcommand{\Dmat}{{\bf D}}
\newcommand{\Amat}{{\bf A}}
\newcommand{\Nmat}{{\bf N}}
\newcommand{\Hmat}{{\bf H}}
\newcommand{\Rmat}{{\bf R}}
\newcommand{\Emat}{{\bf E}}

\newcommand{\Vmat}{{\bf V}}
\newcommand{\Imat}{{\bf I}}
\newcommand{\Xmat}{{\bf X}}

\newcommand{\Wmat}{{\bf W}}
\newcommand{\Pmat}{{\bf P}}
\newcommand{\Qmat}{{\bf Q}}
\newcommand{\bx}{{\bf x}}

\newcommand{\bv}{{\bf v}}

\newcommand{\bn}{{\bf n}}
\newcommand{\bP}{{\bf P}}
\newcommand{\bz}{{\bf z}}
\newcommand{\by}{{\bf y}}

\newcommand{\bu}{{\bf u}}

\newcommand{\bF}{{\bf F}}
\newcommand{\bU}{{\bf U}}

\DeclareMathOperator*{\argmin}{argmin}

\newtheorem{theorem}{Theorem}
\newtheorem{lemma}{Lemma}
\newtheorem{corollary}{Corollary}

\begin{document}

	%
	\title{Scene-Adapted Plug-and-Play Algorithm with Guaranteed Convergence: Applications to Data Fusion in Imaging}
	%
	%
	%
	
	\author{Afonso~M.~Teodoro,~\IEEEmembership{Student Member,~IEEE,}
		Jos\'{e}~M.~Bioucas-Dias,~\IEEEmembership{Fellow,~IEEE,}
		and~M\'{a}rio~A.~T.~Figueiredo,~\IEEEmembership{Fellow,~IEEE}
		\thanks{Afonso~M.~Teodoro, Jos\'{e}~M.~Bioucas-Dias, and~M\'{a}rio~A.~T.~Figueiredo are with Instituto de Telecomunica\c{c}\~{o}es and Instituto Superior T\'{e}cnico, 1049-001, Lisboa, Portugal e-mail: \{afonso.teodoro, jose.bioucas, mario.figueiredo\}@tecnico.ulisboa.pt.}
		\thanks{This work was partially supported by the {\it Funda\c{c}\~ao para a Ci\^encia e Tecnologia} (FCT), grants UID/EEA/5008/2013 and BD/102715/2014.}}
	\maketitle
	\begin{abstract}
		The recently proposed \textit{plug-and-play} (PnP) framework allows leveraging recent developments in image denoising to tackle other, more involved, imaging inverse problems. In a PnP method,  a black-box denoiser is  \textit{plugged} into an iterative algorithm, taking the place of a formal denoising step that corresponds to the proximity operator of some convex regularizer.  While this approach offers flexibility and excellent performance, convergence of the resulting algorithm may be hard to analyze, as most state-of-the-art denoisers lack an explicit underlying objective function.  In this paper, we propose a PnP approach where a scene-adapted prior (\textit{i.e.}, where the denoiser is  targeted to the specific scene being imaged) is plugged into ADMM (alternating direction method of multipliers), and prove convergence  of the resulting algorithm. Finally, we apply the proposed framework in two different imaging inverse problems: hyperspectral sharpening/fusion and image deblurring from blurred/noisy image pairs.
	\end{abstract}
	
	\section{Introduction}
	\label{sec:intro}
	
	Image denoising is one of the most studied problems at the core of image processing. It has been questioned if it is a solved problem, given the excellent performance of several state-of-the-art methods, which seem to be approaching some sort of theoretical limit \cite{Chatterjee}. The answer to this question is negative: denoising is still a very active area of research, not only in itself, but also as a building block for many other tasks. In fact, most algorithms for imaging inverse problems (\textit{e.g.,} deblurring, reconstruction, or super-resolution) include a step that is formally equivalent to a denoising operation. 
		
	A possible approach to leverage state-of-the-art denoising methods in other inverse problems is to generalize/extend their underlying rationale to these other tasks \cite{danielyan,Mignotte,Papyan2016}. The main drawbacks of this approach are that, most often, it is strongly problem-dependent (\textit{i.e.}, different problems require different generalizations) and it may lead to computationally hard problems. The recently proposed \textit{plug-and-play} (PnP) approach aims at sidestepping  these drawbacks and still be able to capitalize on state-of-the-art denoisers. PnP  methods work by plugging a black-box denoiser into an iterative algorithm that address the inverse problem in hand. These algorithms typically include the \textit{proximity operator} \cite{Bauschke}  of a convex regularizer, which is formally equivalent to a denoising operation (under Gaussian white noise), with the convexity of the regularizer playing a central role in the convergence properties of the algorithm. The PnP approach replaces the proximity operator with some black-box denoiser \cite{Venkatakrishnan,Sreehari}, the key advantage being that the same denoiser can be used for different problems, since its use is fully decoupled from the particularities of the inverse problem being addressed.
	
	Replacing the \textit{proximity operator} with some denoiser makes the convergence  of the resulting algorithm  difficult to analyze. While the original algorithm optimizes an objective function, the PnP version lacks an explicit objective function. In the original PnP scheme \cite{Venkatakrishnan,Sreehari}, the iterative algorithm in which the denoiser is plugged-in is the \textit{alternating direction method of multipliers} (ADMM) \cite{boyd,AfonsoCSALSA2010}. PnP-ADMM allows using an arbitrary denoiser to regularize an imaging inverse problem, such as deblurring or super-resolution, tackled via ADMM, but departing from its standard use, where the denoiser is the proximity operator of a convex regularizer \cite{boyd}. (Another recent approach, termed \textit{regularization by denoising}, follows a different route by building an explicit regularizer based on a denoiser \cite{RED}.) The fact that a denoiser, which may lack an explicit objective function or a closed-form, is plugged into ADMM begs three obvious questions \cite{chan,Venkatakrishnan,Sreehari}: 
	\begin{itemize}
		\item[\textbf{(a)}] is the resulting algorithm guaranteed to converge? 
		\item[\textbf{(b)}] if so, does it converge to a (global or local) optimum of some objective function? 
		\item[\textbf{(c)}] if so, can we identify this function? 
	\end{itemize}
    In this paper, we adopt a patch-based denoiser using  \textit{Gaussian mixture models} (GMM) \cite{Teodoro2015,ZoranWeiss,yu}, which allows us to provide positive answers to these three questions.  

	
	As proposed in our earlier work \cite{Teodoro2016}, a GMM can be adapted/specialized to specific classes of images; the rationale is that a denoiser based on a class-adapted prior is able to capture the characteristics of that class better than a general-purpose one. Here, we take this adaptation one step further, by using scene-adapted priors. As the name suggests, the model is no longer learned from a set of images from the same class, but from one or more images of the same scene being restored/reconstructed. Hyperspectral sharpening \cite{review, simoes, qwei, Teodoro2017} and image deblurring from noisy/blurred image pairs \cite{lim, yuan} are the two data fusion problems that may leverage such priors and that we consider in this paper.
	
	In summary, our contributions are the following.
		\begin{itemize}
	            \item[\textbf{(i)}] We propose plugging a GMM-based denoiser into an ADMM algorithm; this GMM-based denoiser is a modified version of a patch-based MMSE (\textit{minimum mean squared error}) denoiser  \cite{Teodoro2015}; the modification consists in fixing the posterior weights of the MMSE denoiser.		
				\item[\textbf{(ii)}] We prove that the resulting PnP-ADMM algorithm is guaranteed to converge to a global minimum of a cost function, which we explicitly identify.
				\item[\textbf{(iii)}] We apply the proposed framework to the two fusion problems mentioned above, showing that it yields results competitive with the stat-of-the-art.
\end{itemize}

	An earlier version of this work was published in \cite{Teodoro2017-2}. With respect
	to that publication, this one includes more detailed and general results and proofs, and the application to noisy/blurred image pairs, which was not considered.
	
	The paper is organized as follows. Section~\ref{sec:admm} briefly reviews ADMM, which is the basis of the proposed algorithm, and its PnP version. Section~\ref{sec:prox} studies conditions for the denoiser to be a proximity operator, focusing on linear denoisers and explicitly identifying the underlying objective function. Section~\ref{sec:gmmden} proposes a GMM-based denoiser and shows that it satisfies the conditions for being a proximity operator. Sections \ref{sec:HSapp} -- \ref{sec:results} present the two instantiations of the proposed method and the corresponding experimental results. Finally, Section~\ref{sec:conclusion} concludes the paper.
	
	

	\section{Alternating Direction Method of Multipliers}\label{sec:admm}
	\subsection{From Basic ADMM to SALSA}
	Although dating back to the 1970's \cite{gabay}, ADMM has seen a surge of interest in the last decade, as a flexible and efficient optimization tool, widely used in imaging problems, machine learning, and other areas \cite{eckstein, boyd}. One of the canonical problem forms addressed by ADMM is 
	\begin{equation}
	\min_\bx \;  f( \bx) + g(\Hmat  \bx), \label{eq:admmcan} \\
	\end{equation}
	where  $f$ and $g$ are closed (equivalently, lower semi-continuous -- l.s.c. \cite{Bauschke}), proper\footnote{A convex function $f:\mathbb{R}^n\rightarrow \bar{\mathbb{R}}$ is called \textit{proper} if its \textit{domain} is not empty: $\mbox{dom}(f) \equiv \{\bx\in\mathbb{R}^n: \; f(\bx) < +\infty\} \neq \emptyset$.}, convex functions, and $\Hmat$ is a full column rank matrix of appropriate dimensions \cite{eckstein}. Each iteration of ADMM for \eqref{eq:admmcan} is as follows (with the superscript $(\cdot)^{(k)}$ denoting the iteration counter):
	\begin{align}
	\bx^{(k+1)} & =  \arg\min_\bx  f(\bx) + \frac{\rho}{2} \bigl\| \Hmat \bx -\bv^{(k)} - \bu^{(k)}\bigr\|_2^2   \label{eq:ineq1}\\
	\bv^{(k+1)} & =  \arg\min_\bv   g(\bv) + \frac{\rho}{2} \bigl\| \Hmat \bx^{(k+1)} - \bv -  \bu^{(k)}\bigr\|_2^2 , \label{eq:ineq2}\\
	\bu^{(k+1)} & =  \bu^{(k)} -  \Hmat \bx^{(k+1)} + \bv^{(k+1)} ,  \label{eq:ineq3}
	\end{align}
	where $\bu^{(k)}$ are the (scaled) Lagrange multipliers at iteration $k$, and $\rho > 0$ is the penalty parameter.
	
	ADMM  can be directly applied to problems involving the sum of $J$ closed, proper, convex functions, composed with linear operators,
	\begin{equation}
	\min_{\bx } \sum_{j=1}^J g_j (\Hmat_{j}\bx) \label{eq:sum1J}
	\end{equation}
	by casting it into the form \eqref{eq:admmcan} as follows: 
	\begin{equation*}
	f(\bx) = 0, \hspace{0.25cm} \Hmat = \begin{bmatrix} \Hmat_{1} \\ \vdots \\ \Hmat_{J} \end{bmatrix}, \hspace{0.25cm} \bv = \begin{bmatrix} \bv_1 \\ \vdots \\ \bv_{J} \end{bmatrix}, \hspace{0.25cm} g(\bv) =  \sum_{j=1}^J g_j (\bv_j).
	\end{equation*}
	The resulting instance of ADMM (called SALSA--\textit{split augmented Lagrangian shrinkage algorithm} \cite{AfonsoCSALSA2010}) is
	\begin{eqnarray}
	\bx^{(k+1)} & =  &\arg\min_\bx  \sum_{j=1}^J \|  \Hmat_j \bx - \bv_j^{(k)}  -\bu_j^{(k)}  \|_2^2  \label{eq:ineq1b}\\
	\bv_1^{(k+1)} & =  &\arg\min_\bz   g_1(\bv) + \frac{\rho}{2} \bigl\| \Hmat_1 \bx^{(k+1)} - \bz -  \bu_1^{(k)}\bigr\|_2^2 , \label{eq:ineq2b}\\
	\vdots  & &  \vdots \nonumber \\
	\bv_J^{(k+1)} & = & \arg\min_\bz   g_J(\bv) + \frac{\rho}{2} \bigl\| \Hmat_J \bx^{(k+1)} - \bz -  \bu_J^{(k)}\bigr\|_2^2 , \label{eq:ineq2c}\\
	\bu^{(k+1)} & =  &\bu^{k} - \Hmat \bx^{(k+1)} + \bv^{(k+1)}, \label{eq:ineq3b}
	\end{eqnarray}
	where (as for $\bv^{(k)}$), $\bu^{(k)} = \bigl[ \bigl(\bu_1^{(k)}\bigr)^T, \, \dots, \, \bigl(\bu_J^{(k)}\bigr)^T\bigr]^T$. Computing $\bx^{(k+1)}$ requires solving a linear system (equivalently, inverting a matrix):
	\begin{equation}
	\bx^{(k+1)} = \Bigl( \sum_{j=1}^{J}  \Hmat_j^T \Hmat_j \Bigr)^{-1} \Bigl(  \sum_{j=1}^{J}  \Hmat_j^T (\bv_j^{(k)}  + \bu_j^{(k)}) \Bigr),
	\end{equation}
	whereas computing $\bv_j^{(k+1)}$ (for $j=1,...,J$) requires applying the proximity operator of the corresponding $g_j$:
	\begin{equation}
    \bv_j^{(k+1)} = \mbox{prox}_{g_j /\rho} \Bigl( \Hmat_j \bx^{(k+1)}  -  \bu_j^{(k)}  \Bigr).
	\end{equation} 
	Recall that the proximity operator of a convex l.s.c. function $\phi$, computed at some point $\bz$, is given by
   \begin{equation}
   \mbox{prox}_{\phi}({\bf z}) = \underset{\textbf{x}}{\argmin} \frac{1}{2}\| {\bf x - z}\|_2^2 + \phi({\bf x}),\label{eq:prox}
  \end{equation}
and is guaranteed to be well defined due to the coercivity and strict convexity of the term $\frac{1}{2}\| {\bf x - z}\|_2^2$ \cite{Bauschke}.
	
	\subsection{Convergence of {ADMM} Algorithm}\label{ssec:conv}
	
	We present a simplified version of a classical theorem on the convergence of ADMM, proved in the seminal paper of Eckstein and Bertsekas \cite{eckstein}. The full version of the theorem allows for inexact solution of the optimization problems in the iterations of ADMM, while this version, which we will use below, assumes exact solutions. For SALSA, being a particular instance of ADMM, the application of the theorem is trivial.
	
	\begin{theorem}[Eckstein and Bertsekas \cite{eckstein}]\label{th:admm}
		Consider a problem of the form \eqref{eq:admmcan}, where $\Hmat$ has full column rank, and $f:\mathbb{R}^n\rightarrow \bar{\mathbb{R}} = \mathbb{R}\cup \{+\infty\}$ and $g:\mathbb{R}^m \rightarrow \bar{\mathbb{R}}$ are closed, proper, convex functions; let $\bv_0, \bu_0 \in \mathbb{R}^m$, and $\rho > 0$ be given. If the sequences $(\bx^{(k)})_{k = 1, 2,\dots}$, $(\bv^{(k)})_{k = 1, 2,\dots }$, and $(\bu^{(k)})_{k = 1,2,\dots}$ are generated according to  \eqref{eq:ineq1}, \eqref{eq:ineq2}, \eqref{eq:ineq3}, then  $(\bx^{(k)})_{k = 1,2,\dots}$ converges to a solution of \eqref{eq:admmcan}, $\bx^{(k)} \rightarrow \bx^{*}$, if one exists. Furthermore, if a solution does not exist, then at least one of the sequences $(\bv^{(k)})_{k = 1, 2,\dots }$ or $(\bu^{(k)})_{k = 1,2,\dots}$ diverges.
	\end{theorem}

	\subsection{Plug-and-Play ADMM}\label{sec:pnp}
	The PnP framework \cite{Venkatakrishnan} emerged from noticing that sub-problem \eqref{eq:ineq2} in ADMM (equivalently, \eqref{eq:ineq2b}-\eqref{eq:ineq2c} in SALSA), which correspond to the application of a proximity operator, can be interpreted as a denoising step. In fact, the definition of proximity operator in \eqref{eq:prox} is formally equivalent the \textit{maximum a posteriori} (MAP) estimate of some unknown $\bx$ from noisy observations $\bz$ (under additive white Gaussian noise with unit variance), under a prior proportional to $\exp(-\phi(\bx))$. Armed with this observation, the goal of the PnP approach is to capitalize on recent developments in image denoising, by replacing a proximity operator by a state-of-the-art denoiser. However, the current state-of-the-art denoising methods do not correspond (at least, explicitly) to proximity operators of convex functions. Most of them exploit image self-similarity by decomposing the image into small patches, {\it i.e.}, they rely on a \textit{divide and conquer} principle, where instead of dealing with the noisy image as a whole, small (overlapping) patches are extracted, denoised independently (using, \textit{e.g.}, collaborative filtering, Gaussian mixtures, dictionaries, ...) and then put back in their locations \cite{aharon, dabov,Buades,Teodoro2015,ZoranWeiss,zoran}. These denoisers often lack a closed-form expression, making the convergence analysis of the resulting algorithm much more difficult. In the following section, we focus on the key properties of the denoising functions needed to satisfy Theorem~\ref{th:admm}, temporarily leaving aside the operator $\Hmat$, as it is problem-dependent.
	
	\section{When is a Denoiser a Proximity Operator?}\label{sec:prox}
	\subsection{Necessary and Sufficient Conditions}
	If the adopted denoiser in the PnP-ADMM scheme can be shown to correspond to the proximity operator of some l.s.c. convex function (although maybe not explicitly derived as such), then nothing really changes. The following classical result by Moreau \cite{moreau1965} (which was used by Sreehari \textit{et al} \cite{Sreehari} in this context) gives necessary and sufficient conditions for some mapping to be a proximity operator.
	\begin{lemma}[Moreau \cite{moreau1965}]  
	A function $p:\mathbb{R}^n \rightarrow \mathbb{R}^n$ is the proximity operator of some function $\phi:\mathbb{R}^n\rightarrow \bar{\mathbb{R}}$ if and only if it is non-expansive and it is a sub-gradient of some convex function $\varphi$, \textit{i.e.}, if and only if,  for any $\bx,\by\in\mathbb{R}^n$,
	 \begin{equation}
	 \Vert p(\bx)-p(\by) \Vert^2 \leq \Vert \bx - \by \Vert^2 \hspace{0.7cm}\mbox{and} \hspace{0.7cm} p(\bx) \in \partial \varphi(\bx),
	 \end{equation}
	 where $\partial\varphi(\bx) = \{ \bz \in\mathbb{R}^n: \varphi(\bx') - \varphi(\bx) \geq \langle \bz,\bx'-\bx\rangle, \; \forall \bx' \}$ denotes the sub-differential of $\varphi$ at $\bx$ \cite{Bauschke}.\label{lem:moreau}
	\end{lemma} 
	 
	 In practice, given some black-box denoiser that performs a complex sequence of operations (\textit{e.g.}, the state-of-the-art BM3D \cite{dabov}) it may be difficult, if not impossible, to verify that it satisfies the conditions of Lemma \ref{lem:moreau}. In the next section, we study a class of denoisers that verify the conditions of Lemma \ref{lem:moreau}, and go one step further by identifying the function of which the denoiser is the proximity operator. This second step is important to study the existence and uniqueness of solutions of the underlying optimization problem.
	
	\subsection{Denoisers as Linear Filters}
	Arguably, the simplest denoising methods estimate each pixel as a linear combination of the noisy pixels, with weights that depend themselves on these noisy pixels. Mathematically, given a noisy image $\by = \bx + \bn$, where $\bx$ is the underlying clean image and $\bn$ is noise, these methods produce an estimate, $\hat{\bx}$, according to
	\begin{equation}
	\hat{\bx} = p({\bf y}) = \Wmat ({\bf y})  \;  {\bf y}, \label{eq:linFilt}
	\end{equation}
	where $\Wmat({\bf y})$ is a square matrix, which is a function of $ {\bf y}$. Many methods, including the classical \textit{non-local means} (NLM \cite{Buades}) can be written in this form. Sreehari \textit{et al} \cite{Sreehari} have studied this class of denoisers, with $\Wmat({\bf y})$ being a symmetric and doubly-stochastic matrix. This is the case of symmetrized NLM \cite{Venkatakrishnan, chan2}, patch-based sparse representations on a dictionary with a fixed support \cite{qwei}, and patch-based approximate MAP estimation under a GMM prior \cite{Sulam}. 
	
	In this paper, we study a different class of denoisers, where the matrix is fixed, $\Wmat({\by}) = \Wmat$, but it only needs to be symmetric, positive semi-definite (p.s.d.), and have spectral radius less or equal to 1 (conditions that are satisfied by any symmetric doubly-stochastic matrix). As shown below, the patch-based GMM-denoiser proposed in this paper satisfies these conditions.
	
	Before proceeding,
	 we introduce some notation. Let $S(\Wmat)$ denote the \textit{column span} of $\Wmat$. Let $\iota_C:\mathbb{R}^n\rightarrow \bar{\mathbb{R}}$ denote the \textit{indicator function} of the convex set $C\subset\mathbb{R}^n$, \textit{i.e.}, $\iota_C(\bx) = 0$, if $\bx\in C$, and $\iota_C(\bx) = +\infty$, if $\bx\not\in C$.  	The following theorem is proved in Appendix A.
	

	\begin{theorem}\label{lem:1}
		Let $p:\mathbb{R}^n \rightarrow \mathbb{R}^n$ be defined as $p(\by) = \Wmat\by$, where $\Wmat\in\mathbb{R}^{n\times n}$ is a symmetric, p.s.d. matrix, with maximum eigenvalue satisfying $0 < \lambda_{\mbox{\scriptsize max}} \leq 1$. Let $\Wmat = \Qmat \Lambda \Qmat^T$ be the eigendecomposition of $\Wmat$, where $\Lambda = \mbox{diag}(\lambda_1 = \lambda_{\mbox{\scriptsize max}},...,\lambda_n)$ contains the eigenvalues, sorted in decreasing order; the first $r = \mbox{rank}(\Wmat)$ are non-zero and the last  $(n-r)$ are zeros. Let $\bar{\Lambda}$ be the $r\times r$ diagonal matrix with the non-zero eigenvalues and $\bar{\Qmat}$ the matrix with the first $r$ columns of $\Qmat$ (an orthonormal basis for $S(\Wmat))$. Then, 
		\begin{itemize}
			\item[\textbf{(i)}] $p$ is the proximity operator of a convex function $\phi$; 
			\item[\textbf{(ii)}] $\phi$ is the following closed, proper, convex function:
					\begin{equation}
	\phi(\bx) = \iota_{S(\Wmat)} (\bx) + \frac{1}{2}\bx^T \bar{\Qmat} ( \bar{\Lambda}^{-1} - \Imat ) \bar{\Qmat}^T \bx. \label{eq:phi_def}
	    \end{equation}
	\end{itemize}\label{th:prox}
	\end{theorem}

	\vspace{0.25cm}	
	We present (and prove) claims (i) and (ii) separately to highlight that (as done by  Sreehari \textit{et al} \cite{Sreehari}), by using Lemma \ref{lem:moreau}, it is possible to show that $p$ is a proximity operator without the need to explicitly obtain the underlying function. 
	
	\section{GMM-Based Denoising}\label{sec:gmmden}
	In this paper, we adopt a GMM patch-based denoiser to plug into ADMM. The reasons for this choice are threefold: \textbf{(i)} it has been shown that GMM are good priors for clean image patches \cite{Teodoro2015,ZoranWeiss, yu}; \textbf{(ii)} a GMM can be learned either from an external dataset of clean images \cite{ZoranWeiss, yu}, or directly from noisy patches \cite{Teodoro2015}; \textbf{(iii)} the ease of learning a GMM  from an external dataset  opens the door to  class-specific priors, which, naturally, lead to superior results since they capture the characteristics of the image class in hand better than a general-purpose image model \cite{Teodoro2016,luo}.
	
	
	
	Leaving aside, for now, the question of how to learn the parameters of the model,  we consider a GMM prior\footnote{$\mathcal{N}(\bx ; \mu,{\bf C})$ denotes a Gaussian probability density function of mean $\mu$ and covariance $\bf C$, computed at $\bx$.}
	\begin{equation}
	p({\bf x}_i ) = \sum_{j=1}^K \alpha_j \mathcal{N}({\bf x}_i ; \bmu_j, {\bf C}_j),
	\end{equation}
	for any patch ${\bf x}_i$ of the image to be denoised. A current practice in patch-based image denoising is to treat the average of each patch separately: the average of the noisy patch is subtracted from it, the resulting zero-mean patch is denoised, and the average is added back to the estimated patch. The implication of this procedure is that, without loss of generality, we can assume that the  components of the GMM have  zero mean ($\mu_j = 0$, for $j=1,...,K$). 
	
	Given this GMM prior and a noisy version $\by_i$ of each patch, where $\by_i = \bx_i + \bn_i$, with $\bn_i \sim \mathcal{N}(0,\sigma^2)$, the MMSE estimate of $\bx_i$ is given by
	\begin{equation}
	\hat{\bf x}_i = \sum_{j=1}^K \beta_j({\bf y}_i ) \; {\bf v}_j({\bf y}_i),\label{eq:MMSE}
	\end{equation}
	where
	\begin{equation}
	{\bf v}_j ({\bf y}_i) =  \Cmat_j \Bigl( {\bf C}_j +  \sigma^2 {\bf I} \Bigr)^{-1}  {\bf y}_i, \label{eq:v_m}
	\end{equation}
	and 
	\begin{equation}
	\beta_j({\bf y}_i) = \frac{\alpha_j \; \mathcal{N}({\bf y}_i ; 0, {\bf C}_j + \sigma^2 \, {\bf I})
	}{\sum_{k=1}^K \alpha_j \; \mathcal{N}({\bf y}_i ; 0, {\bf C}_k + \sigma^2 \, {\bf I})}. \label{eq:beta}
	\end{equation}
	Notice that  $\beta_j({\bf y}_i)$ is  the posterior probability that the \textit{i}-th patch was generated by the \textit{j}-th component of the GMM, and ${\bf v}_j({\bf y}_i)$ is the MMSE estimate of the \textit{i}-th patch if it was known that it had been generated by the $j$-th component.

	As is standard in patch-based denoising, after computing the MMSE estimates of the patches, they are returned to their location and combined by straight averaging. This corresponds to solving the following optimization problem
	\begin{align}
	\widehat{\bx }  \in  \underset{\bx \in \mathbb{R}^n }{\argmin}  \quad  \sum_{i = 1}^{N} \Vert \hat{\bf x}_i   - \bP_i \bx \Vert_2^2, \label{eq:wholeopt}
	\end{align}
	where $\bP_i \in \{0,1\}^{n_p \times n}$ is a binary matrix that extracts the $i$-th patch from the image (thus $\bP_i^T$ puts the patch back into its place), $N$ is the number of patches, $n_p$ is the number of pixels in each patch, and $n$ the total number of image pixels. The solution to \eqref{eq:wholeopt} is
	\begin{equation}
	\hat{\bx} = \Bigl( \sum_{i = 1}^{N} \bP_i^T \bP_i \Bigr)^{-1}\Bigl( \sum_{i = 1}^{N} \bP_i^T \hat{\bf x}_i \Bigr) = \frac{1}{n_p}\,  \sum_{i = 1}^{N} \bP_i^T \hat{\bf x}_i, \label{eq:wholex}
	\end{equation}
	assuming the patches are extracted with unit stride and periodic boundary conditions, thus every pixel belongs to $n_p$ patches and $\sum_{i = 1}^{N} \bP_i^T \bP_i = n_p \Imat$.

	\subsection{From Class-Adapted to Scene-Adapted Priors}\label{ssec:class}
	Class-adapted priors have been used successfully in image deblurring, if the image to be deblurred is known to belong to a certain class, {\it e.g.}, a face or a fingerprint, and a prior adapted to that class is used \cite{Teodoro2016}.
	Here, we take this adaptation to an extreme, which is possible if we have two types of data/images of the same scene, as is the case with some fusion problems. Instead of using a dataset of other images from the same class (\textit{e.g.}, faces), we leverage one type of data to learn a model that is adapted to the specific scene being imaged, and use it as a prior. The underlying assumption is that the two types of data have similar spatial statistical properties, which is reasonable, as the scene being imaged is the same. In Sections~\ref{sec:HSapp} and \ref{sec:deb}, we apply this idea to tackle two data fusion problems, namely hyperspectral (HS) sharpening \cite{review}, and image deblurring from a noisy/blurred image pair \cite{yuan}. 
	
	In HS sharpening, the goal is to fuse HS data of low spatial resolution, but high spectral resolution, with multispectral (MS) or panchromatic (PAN) data of high spatial resolution, but low spectral resolution, to obtain an HS image with  high spatial and spectral resolutions. This can be seen as a super-resolution problem, where the HS image is super-resolved with the help of a spatial prior obtained from the MS image of the same scene. More specifically, we resort to a GMM prior, as described in the previous section, learned from (patches of the bands of) the observed MS or PAN image using the classical \textit{expectation-maximization} (EM) algorithm, and then leverage this prior to sharpen the low resolution HS bands. 

	In the second problem, the goal is to restore a blurred, but relatively noiseless, image, with the help of another image of the same scene, which is sharp but noisy. The noisy sharp image is used to learn a spatial prior, adapted to the particular scene, which is then used as a prior to estimate the underlying sharp and clean image. Again, this is carried out by learning a GMM from the patches of the noisy, but sharp, image.

	In light of the scene-adaptation scheme just described, we further claim that, it not only makes sense to use a GMM prior learned from the observed sharp images, but it also makes sense to keep the weights $\beta_j$ of each patch in the target image, required to compute \eqref{eq:MMSE}, equal to the values for the corresponding observed patch, obtained during training.  In other words, we use the weights as if we were simply denoising the sharp image, with a GMM trained from its noisy patches. In fact, another instance of this idea has been previously used, based on sparse representations on learned dictionaries \cite{qwei}. We stress that this approach is only reasonable in the case of image denoising or image reconstruction using scene-adapted priors, where the training and target images are of the same scene, and thus have similar spatial structures. As will be shown in the next section, this modification will play an important role in proving convergence of the proposed PnP-ADMM algorithm.

%
%
%
%
	\subsection{Analysis of the Scene-adapted GMM-based Denoiser}
	\label{sec:GMMdenoiser}
	Consider the role of the denoiser: for a noisy input argument $\by \in \mathbb{R}^{n}$, it produces an estimate $\hat{\bx} \in \mathbb{R}^{n}$, given the parameters of the GMM, $\theta = \{\Cmat_1, \dots, \Cmat_K, \beta_1^1, \dots, \beta_K^n\}$, and the noise variance $\sigma^2$. As the GMM-denoiser is patch-based, the algorithm starts by extracting (overlapping, with unit stride) patches, followed by computing their estimates using \eqref{eq:MMSE}. Recall that, with the scene-adaptation scheme, parameters $\beta$ no longer depend on the noisy patches, and thus we may write compactly,
	\begin{align}
	\label{eq:linMMSE}
	\hat{\bf x}_i &= \sum_{j=1}^K \beta_j^{i} \; {\bf C}_j \Bigl(  {\bf C}_j + {\sigma^2\, \bf I} \Bigr)^{-1} \by_i  
	= \bF_i \, \by_i = \bF_i \, \bP_i\, \by,
	\end{align}
	where $\bP_i$ is the same matrix that appears in \eqref{eq:wholeopt}, thus $\by_i = \bP_i\, \by$. Furthermore, using \eqref{eq:wholex}, 
	\begin{equation}
	\hat{\bx} = \underbrace{\frac{1}{n_p} \sum_{i=1}^N \Pmat_i^T   \bF_i \Pmat_i}_{\Wmat} \; {\bf y} = \Wmat  \;  {\bf y}, \label{eq:defW}
	\end{equation}
	which shows that the GMM-based MMSE denoiser, with fixed weights, is a linear function of the noisy image. Matrix $\Wmat$, of course, depends on the parameters of the method: $\sigma^2$, $n_p$, and the parameters of the GMM, but it is fixed throughout the iterations of ADMM. Recall that the weights $\beta_j^i$ and covariances $\Cmat_j$  are obtained during training, and $\sigma^2$ is assumed known, thus is fixed. 
	The following Lemma summarizes the key properties of $\Wmat$, when using the scene-adapted GMM denoiser. The proof can be found in Appendix B.
	
	\begin{lemma}\label{lem:22}
		Matrix $\Wmat$ is symmetric, positive semi-definite, with maximum eigenvalue no larger than 1.
	\end{lemma}
	
	Lemma \ref{lem:22} and Theorem \ref{th:prox} yield the following corollary.
	
	\begin{corollary}\label{lem:2}
The GMM patch-based denoiser defined in \eqref{eq:defW} corresponds to the proximity operator of the closed, proper, convex function defined in \eqref{eq:phi_def}.
	\end{corollary}
	
	Notice that Lemma \ref{lem:22} and Corollary \ref{lem:2} hinge  on the fact that the denoiser uses fixed weights $\bbeta_m^i$, rather than being a pure MMSE denoiser. In fact, the simple univariate example in Fig.~\ref{fig:nonlinear} shows that an MMSE denoiser may not even be non-expansive.
	This figure plots the MMSE estimate $\hat{\bx}$ in \eqref{eq:MMSE}, as a function of the noisy observed  $\by$, under a GMM prior with two zero-mean components (one with small and one with large variance). Clearly, the presence of regions of the function with derivative larger than 1 implies that it is not non-expansive.

	\begin{figure}
		\begin{minipage}[b]{.45\linewidth}
			\centering
			\subfloat[]{\includegraphics[width=4cm]{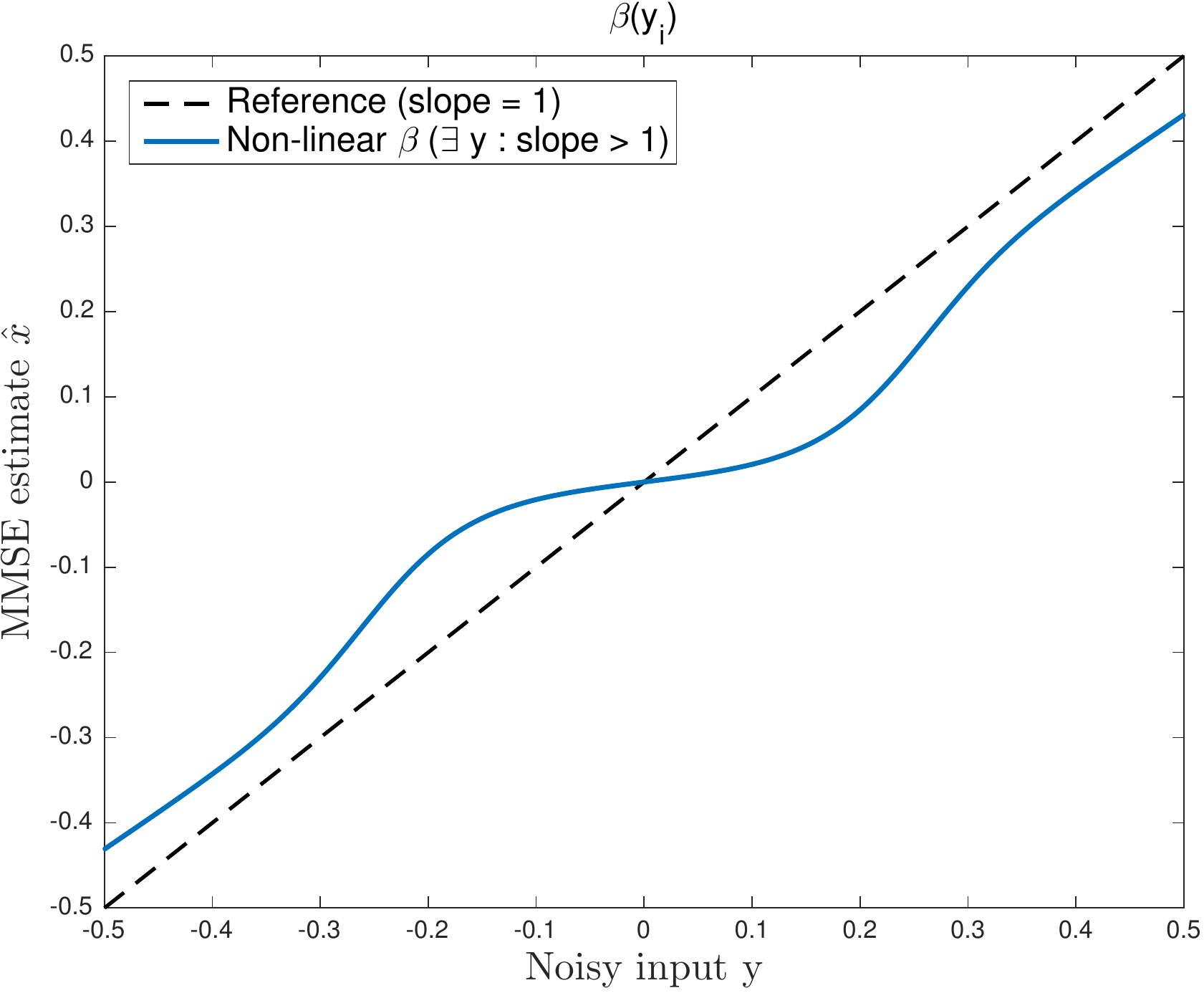}\label{fig:nonlinear}}
		\end{minipage}
		\hfill
		\begin{minipage}[b]{0.45\linewidth}
			\centering
			\subfloat[]{\includegraphics[width=4cm]{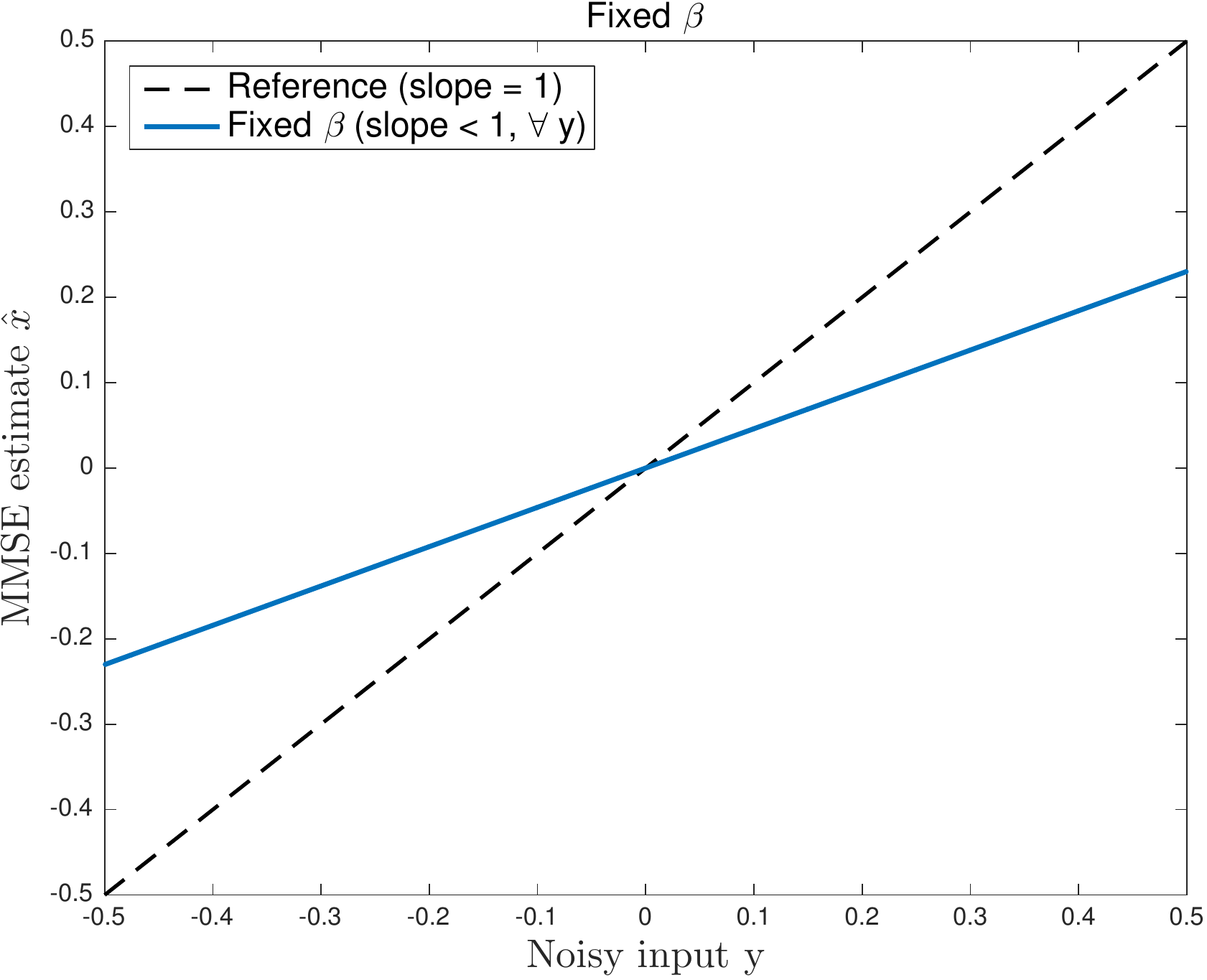}\label{fig:fixed}}
		\end{minipage}
		\caption{Denoiser expansiveness example: (a) Non-linear weights, $\beta(\by_i)$; (b) Fixed weights, $\beta$.}
		\label{fig:exp}
	\end{figure}
	
	\subsection{Convergence Proof}
	
	Corollary \ref{lem:2} allows proving convergence of the proposed algorithm. This result is formalized in Corollary \ref{th:conv}.
	
	\begin{corollary}\label{th:conv}
		Consider the SALSA algorithm (in \eqref{eq:ineq1b}--\eqref{eq:ineq3b}) applied to a problem of the form \eqref{eq:sum1J}, with  $\Hmat_j = \Imat$, for at least one $j\in\{1,...,j\}$. Consider also that $g_1, \dots, g_{J-1}$ are proper closed convex functions, and $g_J = \phi$, where  $\phi$ is as defined in \eqref{eq:phi_def}. Furthermore, assume that the intersection of the domains of $g_1,...,g_{J}$ is not empty. Then, the algorithm converges.
	\end{corollary}	
	
 \begin{proof}
 	The proof amounts to verifying the conditions of Theorem~\ref{th:admm}.  The first condition for convergence is that $\Hmat$ has full column rank. Since 
	\[
	\Hmat = \begin{bmatrix} \Hmat_1^T , \dots , \Hmat_{J}^T \end{bmatrix}^T ,
	\]
	the fact that $\exists j\in\{1,...,J\}\! : \; \Hmat_j = \Imat$ implies that $\Hmat$ has full column rank, regardless of the other blocks. The second condition regards functions $g_1$ to $g_{J}$; while $g_1$ to $g_{J-1}$ are assumed to be closed, proper, and convex, Corollary~\ref{lem:2} guarantees that $g_J = \phi$ is also closed, proper, and convex. Finally, the function being minimized is the sum of $J$ closed, proper, convex functions, such that the intersection of the corresponding domains is not empty; thus, the sum is itself closed, proper, and convex, implying that a minimizer is guaranteed to exist. \end{proof}

		\section{Application to Hyperspectral Sharpening}\label{sec:HSapp}
		\subsection{Formulation}
		
		Hyperspectral (HS) sharpening  is a data fusion problem where the goal is to combine HS data of very high spectral resolution, but low spatial resolution, with  multispectral (MS) or panchromatic (PAN) data of high spatial resolution, but low spectral resolution \cite{review}. This data fusion problem can be formulated as an inverse problem, where the goal is to infer an image $\Zmat$ with both high spatial and spectral resolutions, from the observed data \cite{review,Hardie}. Using compact matrix notation, the forward model can be written as
		\begin{eqnarray}
		\Ymat_h & = & \Zmat \Bmat \Mmat + \Nmat_h,\label{eq:yh} \\
		\Ymat_m & = & \Rmat \Zmat + \Nmat_m, \label{eq:ym}
		\end{eqnarray}
		where:
		\begin{itemize} 
			\item $\Zmat \in \mathbb{R}^{L_h \times n_m}$ is the  image to be inferred, with  $L_h$ bands and $n_m$ pixels (columns index pixels,  rows index the bands);
			\item $\Ymat_h \in \mathbb{R}^{L_h \times n_h}$ denotes the observed HS data, with $L_h$ bands and $n_h$ pixels (naturally, $n_h < n_m$; often, $n_h \ll n_m$)
			\item $\Ymat_m \in \mathbb{R}^{L_m \times n_m}$ is the observed MS/PAN data, with $L_m$ bands and $n_m$ pixels (naturally, $L_m < L_h$; often, $L_m \ll L_h$);
			\item $\Bmat \in \mathbb{R}^{n_m \times n_m}$ represents  a spatial low-pass filter;
			\item $\Mmat \in \mathbb{R}^{n_m \times n_h}$ is the spatial sub-sampling operator;
			\item $\Rmat \in \mathbb{R}^{L_m \times L_h}$ models the spectral response of the MS or PAN sensor;
			\item $\Nmat_h$ and $\Nmat_m$ represent noise, assumed white,  Gaussian, with known variances $\sigma_h^2$ and $\sigma_m^2$, respectively.
		\end{itemize}
		Eqs. \eqref{eq:yh}-\eqref{eq:ym} model the HS data 	$\Ymat_h$ as a spatially blurred and downsampled version of $\Zmat$, and the MS data $\Ymat_m$ as a spectrally degraded version of $\Zmat$,  via the spectral response matrix $\Rmat$. In this paper, we assume that $\Bmat$ and $\Rmat$ are known, and that the spatial filter $\Bmat$ is cyclic so that we may handle it efficiently using the \textit{fast Fourier transform} (FFT); this assumption could be easily relaxed using the method  in \cite{Almeida2013}. Moreover, other types of noise could also be considered within this formulation \cite{qwei}.
		
		Since the number $L_h  \, n_m$ of unknowns in $\Zmat$ is typically much larger than the number of measurements, $L_h  \, n_h + L_m  \, n_m$, the problem of estimating $\Zmat$ is ill-posed. Current state-of-the-art methods tackle it using a variational formulation (equivalently, \textit{maximum a posteriori }-- MAP) \cite{review}, \textit{i.e.},  seeking an estimate $	\widehat{\Zmat}$ according to
		\begin{equation}
		\widehat{\Zmat}  \in  \underset{\Zmat}{\argmin}   \;  \frac{\Vert \Zmat \Bmat \Mmat  - \Ymat_h \Vert_F^2}{2} +  \frac{ \lambda \Vert \Rmat \Zmat  - \Ymat_m \Vert_F^2}{2} + \tau \;\xi(\Zmat), \label{eq:optimization1}
		\end{equation}
		where $\xi$ is the regularizer (or negative log-prior), while parameters $\lambda$ and $\tau$ control the relative weight of each term.
		
		One way to alleviate the difficulty associated with the very high dimensionality of $\Zmat$ is to exploit the fact that the spectra in a HS image typically live in a subspace of $\mathbb{R}^{L_h}$ of dimension $L_s$ (often with $L_s  \ll L_h$), due to  the high correlation among the different bands \cite{unmixing,review}.
		This subspace can be interpreted as a set of materials existing on the scene, which is often much smaller than the number of bands. There are several methods to identify a suitable subspace, from the classical \textit{singular value decomposition} (SVD) or \textit{principal component analysis} (PCA), to methods specifically designed for HS data \cite{hysime,Nascimento2005}. Let $\Emat = \left[\textbf{e}_1, \dots, \textbf{e}_{L_s}\right] \in \mathbb{R}^{L_h \times L_s}$ contain a basis for the identified subspace, then we may write $\Zmat = \Emat \Xmat$, where each column of $\Xmat \in \mathbb{R}^{L_s\times n_m}$, denoted latent image, is a coefficient representation of the corresponding column of $\Zmat$.  Instead of estimating $\Zmat$ directly, this approach reformulates the problem as that of obtaining an estimate of $\Xmat$, from which $\Zmat$ can be recovered. Replacing $\Zmat$ by its representation $\Zmat = \Emat \Xmat$  in \eqref{eq:optimization1} leads to a new estimation criterion
		\begin{align}
		\widehat{\Xmat}  \in  \underset{\Xmat}{\argmin} \;  \frac{\Vert \Emat \Xmat \Bmat \Mmat \! - \!\Ymat_h \Vert_F^2}{2} \! + \!\frac{ \lambda \Vert \Rmat \Emat \Xmat \! - \!\Ymat_m \Vert_F^2}{2}  + \tau \, \phi(\Xmat), \label{eq:optimization}
		\end{align}
		where $\| \cdot \|_F^2$ denotes the squared Frobenius norm and $\phi$ is the regularizer acting on the coefficients.

		\subsection{SALSA for Hyperspectral Sharpening}
		We follow Sim\~oes \textit{et al.}  \cite{simoes}  and use SALSA to tackle problem \eqref{eq:optimization}. Notice that, while in \eqref{eq:optimization}, the optimization variable is a matrix $\Xmat$, in \eqref{eq:sum1J}, the optimization variable is a vector $\bx$. This is merely a difference in notation, as we can represent matrix $\Xmat$ by its vectorized version $\bx = \mbox{vec}(\Xmat)$, \textit{i.e.}, by stacking the columns of $\Xmat $ into vector $\bx$. In fact, using well-known equalities relating vectorization and Kronecker products, we can  write
		\begin{align}
		\|  \Emat \Xmat \Bmat \Mmat  - \Ymat_h \|_F^2 & =  \left\|\bigl( \Mmat^T \otimes \Emat\bigr) (\Bmat^T \otimes \Imat) \bx - \mbox{vec} (\Ymat_h) \right\|_2^2\\
		\Vert \Rmat \Emat \Xmat  -  \Ymat_m \Vert_F^2 & =   \left\|   \bigl(\Imat \otimes (\Rmat\Emat) \bigr)  \bx - \mbox{vec} (\Ymat_m) \right\|_2^2,
		\end{align} 
		and map \eqref{eq:optimization} into \eqref{eq:sum1J} by letting $J=3$, and
		\begin{align}
		g_1(\bv_1) &= \left\|\bigl( \Mmat^T \otimes \Emat\bigr) \bv_1 - \mbox{vec} (\Ymat_h) \right\|_2^2   \label{eq:g1} 
		 = \Vert \Emat \Vmat_1 \Mmat  - \Ymat_h \Vert_F^2 \\
		g_2(\bv_2) &=  \lambda  \left\|   \bigl(\Imat \otimes (\Rmat\Emat) \bigr) \bv_2 - \mbox{vec} (\Ymat_m) \right\|_2^2 \label{eq:g2} \\ 
		& =   \lambda \Vert \Rmat \Emat \Vmat_2  - \Ymat_m \Vert_F^2, \nonumber \\
		g_3(\bv_3) &=  2 \,\tau \, \phi(\Vmat_3),
		\end{align}
		where $\bv_i = \mbox{vec}(\Vmat_i)$, for $i=1,2,3$,  and 
		\begin{equation}
		\Hmat_1 = (\Bmat^T \otimes \Imat), \hspace{0.5cm}\Hmat_2 = \Imat,  \hspace{0.5cm} \Hmat_3 = \Imat.  \label{Hmatrices}
		\end{equation}
		
		Although the previous paragraph showed that we can arbitrarily switch between matrix and vector representations, in what follows it is more convenient to keep the matrix versions. Each iteration of the resulting instance of ADMM  has the form
		\begin{align}
		& \Xmat^{(k+1)}  =   \arg\min_{\Xmat} \,   \Vert \Xmat \Bmat -  \Vmat^{(k)}_1  -  \Dmat^{(k)}_1 \Vert_F^2 + \Vert \Xmat  -  \Vmat_2^{(k)}  -  \Dmat^{(k)}_2 \Vert_F^2  \nonumber \\
		&  \hspace{2.2cm } + \Vert \Xmat -  \Vmat_3^{(k)} \! - \! \Dmat^{(k)}_3 \Vert_F^2,  \nonumber  \\
		& \Vmat_1^{(k+1)} = \arg \min_{\Vmat_1} \, \,  \Vert \Emat \Vmat_1 \Mmat - \Ymat_h \Vert_F^2 + \rho  \Vert \Xmat^{(k+1)}\Bmat - \Vmat_1 - \Dmat_1^{(k)} \Vert_F^2, \nonumber \\
		& \Vmat_2^{(k+1)}  =\arg \min_{\Vmat_2} \, \, \lambda\, \Vert \Rmat \Emat \Vmat_2 - \Ymat_m \Vert_F^2 + \rho  \Vert \Xmat^{(k+1)} - \Vmat_2 - \Dmat_2^{(k)} \Vert_F^2,  \nonumber \\
		& \Vmat_3^{(k+1)}  =\arg \min_{\Vmat_3} \, \, \phi(\Vmat_3) + \frac{\rho}{2\, \tau} \Vert \Xmat^{(k+1)} - \Vmat_3 - \Dmat_3^{(k)} \Vert_F^2, \label{eq:mpo}\\
		&\Dmat_1^{(k+1)} = \Dmat_1^{(k)} - \left(\Xmat^{(k+1)}\Bmat - \Vmat_1^{(k+1)}\right), \label{D1update}\\
		&\Dmat_2^{(k+1)} = \Dmat_2^{(k)} - \left(\Xmat^{(k+1)} - \Vmat_2^{(k+1)}\right), \label{D2update}\\
		&\Dmat_3^{(k+1)} = \Dmat_3^{(k)} - \left(\Xmat^{(k+1)} - \Vmat_3^{(k+1)}\right) , \label{D3update}
		\end{align}
		where $\rho$ is the so-called \textit{penalty parameter} of ADMM \cite{boyd}. The first three problems are quadratic, thus with closed-form solutions involving matrix inversions  \cite{simoes}:
		\begin{align}
		\Xmat^{(k+1)} & = \label{Xupdate}  \\
		& \hspace{-1cm } \left[ \left(\Vmat_1^{(k)} \! + \!\Dmat_1^{(k)}\right)\Bmat^T + \Vmat_2^{(k)} +  \Dmat_2^{(k)} + \Vmat_3^{(k)} +  \Dmat_3^{(k)} \right]\left[ \Bmat \Bmat^T  \!+ \! 2\, \Imat \right]^{-1}, \nonumber \\
		\Vmat_1^{(k+1)} & = \left[\Emat^T\Emat + \rho\Imat\right]^{-1}\left[\Emat^T\Ymat_h + \rho\left(\Xmat^{(k+1)}\Bmat - \Dmat_1^{(k)}\right)\right]\odot\Mmat \notag \\ & \quad + \left(\Xmat^{(k+1)}\Bmat - \Dmat_1^{(k)}\right)\odot \left(1-\Mmat\right), \label{V1update} \\
		\Vmat_2^{(k+1)} & = \left[\lambda\Emat^T\Rmat^T\Rmat\Emat + \rho\Imat\right]^{-1}\left[\lambda\Emat^T\Rmat^T\Ymat_m + \rho\left(\Xmat^{(k+1)} - \Dmat_2^{(k)}\right)\right], \label{V2update}
		\end{align}
		where $\odot$ denotes entry-wise (Hadamard) product.  The matrix inversion in \eqref{Xupdate} can be computed with $O(n_m \log n_m)$ cost via  FFT, assuming periodic boundary conditions or unknown boundaries  \cite{Almeida2013}. The matrix inversions in \eqref{V1update} and \eqref{V2update} involve matrices of size $L_s \times L_s$; since $L_s$ is typically a number around $10\sim30$, the cost of these inversions is marginal. Moreover, with fixed $\rho$, these two inverses can be pre-computed \cite{simoes}.
		
		\subsection{GMM-Based PnP-SALSA Algorithm}
		Finally, the PnP algorithm is obtained by letting the regularizer $\phi$ be as defined in \eqref{eq:phi_def}, \textit{i.e.},  \eqref{eq:mpo} takes the form
		\begin{equation}
	    \Vmat_3^{(k+1)}  = \Wmat \, \bigl( \Xmat^{(k+1)} \! - \Dmat_3^{(k)}  \bigr) , \label{eq:mpo2}
		\end{equation}
		where $\Wmat$ is as given in \eqref{eq:linMMSE}--\eqref{eq:defW}, with $\sigma^2 = \tau/\rho$. Of course, \eqref{eq:mpo2} is just the formal representation of the denoiser acting on $\bigl( \Xmat^{(k+1)} - \Dmat_3^{(k)}  \bigr)$; matrix $\Wmat$ is never explicitly formed, and the multiplication by $\Wmat$ corresponds to the patch-based denoiser described in Section \ref{sec:gmmden}. The resulting PnP-SALSA algorithm is as shown in Algorithm 1, and its convergence is stated by the following corollary.
				\begin{corollary}
			Algorithm 1 converges.
		\end{corollary}
		
		\begin{proof}
			Convergence of Algorithm 1 is a direct consequence of Corollary \ref{th:conv}, with $J=3$, and $g_3 = \phi$, with $\phi$ as given by \eqref{eq:phi_def}. The condition that at least one of the $\Hmat_j$ matrices is an identity is satisfied in \eqref{Hmatrices}. Functions $g_1$ and $g_2$ (see \eqref{eq:g1} and \eqref{eq:g2})  are quadratic, thus closed, proper, and convex, and their domain is the whole Euclidean space where they are defined; consequently, $\mbox{dom}(g_1)\cap \mbox{dom}(g_2)\cap \mbox{dom}(g_3) \neq\emptyset$, implying that \eqref{eq:optimization} has at least one solution.
		\end{proof}
		
			\begin{algorithm}[thb]
			\SetAlgoLined
			\KwIn{$\Ymat_h$, $\Ymat_m$, $\Rmat$, $\Bmat$, $\Emat$, $\Wmat$, $\rho$, $\lambda$, $\tau$\;}
			\KwOut{$\widehat{\Zmat}$\;}
			Initialization:  $k=0$, $\Dmat_1^{(k)} = \Dmat_2^{(k)} = \Dmat_3^{(k)} = 0$, $\Vmat_1^{(k)} = \Vmat_2^{(k)} = \Vmat_3^{(k)} = 0$\; 
			\Repeat{convergence}{
				Compute $\Xmat^{(k+1)}$ via \eqref{Xupdate}\;
				Compute $\Vmat_1^{(k+1)}$ via \eqref{V1update}\;
				Compute $\Vmat_2^{(k+1)}$ via \eqref{V2update}\;
				Compute  $\Vmat_3^{(k+1)}$ via \eqref{eq:mpo2}\;
				Compute $\Dmat_1^{(k+1)}$ via \eqref{D1update}\;  	
				Compute $\Dmat_2^{(k+1)}$ via \eqref{D2update}\;  	
				Compute $\Dmat_3^{(k+1)}$ via \eqref{D3update}\;
				$k \leftarrow k+1$
			}
			Reconstruct HS high resolution image: $\widehat{\Zmat}= \Emat \Xmat^{(k)}$;
			\caption{PnP-SALSA for Hyperspectral Sharpening}
		\end{algorithm}
		
		\subsection{Complete Hyperspectral Sharpening Algorithm}\label{sec:complete}
		Finally, to obtain a complete HS sharpening algorithm, we still need to specify how matrices $\Emat$ and $\Wmat$ are obtained.
		Matrix $\Emat$ is obtained by applying PCA to all the spectra in the HS data, $\Ymat_h$.
		 
		As mentioned above, matrix $\Wmat$ is not explicitly formed; multiplication by $\Wmat$ corresponds to the patch-based denoiser described in Subsection \ref{sec:GMMdenoiser}, based on the covariance matrices $\Cmat_1,...,\Cmat_K$, and  weights $\beta_j^i$. These covariance matrices are learned from all the patches of all the bands of $\Ymat_m$, using the EM algorithm. With $\by_{m,p}^i$ denoting the (vectorized) $i$-th patch of the $p$-th band of $\Ymat_m$, each E-step of EM computes the posterior probability that this patch was produced by the $j$-the GMM component:
			\begin{equation}
			\beta_{p,j}^i = \beta_j (\by_{m,p}^i) = \frac{\alpha_j \; \mathcal{N}( \by_{m,p}^i ; 0, {\bf C}_j + \sigma_m^2 \, {\bf I})
			}{\sum_{k=1}^K \alpha_k \; \mathcal{N}( \by_{m,k}^i; 0, {\bf C}_k + \sigma_m^2 \, {\bf I})},
			\end{equation} 
			where $\alpha_1,...,\alpha_K$ are the current weight estimates. The M-step updates the weight and covariance estimates. The weights are updated using the standard expression in EM for GMM:
			\[
			\alpha_j \leftarrow \frac{\sum_i \sum_p \beta_{p,j}^i}{\sum_j \sum_i \sum_p \beta_{p,j}^i}, \;\;\mbox{for $j=1,...,K$.}
			\] 
			The covariances update has to be modified to take into account that the patches have noise (see \cite{Teodoro2015}),
			\[
			\Cmat_j \leftarrow \mbox{eigt}\biggl(\frac{\sum_i\sum_p \beta_{p,j}^i (\by_{m,p}^i)(\by_{m,p}^i)^T}{\sum_i\sum_p \beta_{p,j}^i} - \sigma_m^2\Imat\biggr), \;\;\mbox{for $j=1,...,K$,}
			\]
			where $\mbox{eigt}(\cdot)$ is the eigenvalue thresholding operation, which guarantees that $\Cmat_j$ is a p.s.d. matrix after subtracting $\sigma_m^2\Imat$; for a real symmetric matrix argument, $\mbox{eigt}(\Amat) = \Qmat \max(\Lambda,0) \Qmat^T$,
			 where the columns of $\Qmat$ are the eigenvectors of $\Amat$, $\Lambda$ is the diagonal matrix holding its eigenvalues, and the $\max$ operation is applied element-wise.
			
			After convergence of  EM  is declared, the posterior weights of each patch are computed by averaging across bands,
			\begin{equation}
			\beta_{j}^i = \frac{1}{L_m}  \sum_{p=1}^{L_m} \beta_{p,j}^i \label{eq:betaij}
			\end{equation}
			and stored in $\bbeta \in [ 0,1 ]^{K\times N} $ (recall that  $N$ is the number of patches). Naturally, we assume that the patches are aligned across bands, \textit{i.e.}, all the patches  $\by_{m,p}^i$, for $p=1,...,L_m$, are exactly in the same spatial location. The complete method is described in Algorithm 2.
		
			\begin{algorithm}[thb]
			\SetAlgoLined
			\KwIn{$\Ymat_h$, $\Ymat_m$, $\Rmat$, $\Bmat$, $\rho$, $\lambda$, $\tau$, $L_s$, $n_p$}
			\KwOut{$\widehat{\Zmat}$}
			$\Emat \leftarrow \mbox{PCA}(\Ymat_h, L_s)$\;
			$(\Cmat_1,...,\Cmat_K,\bbeta)  \leftarrow  \mbox{EM}(\Ymat_m,n_p,\sigma_m^2)$\;
			Build the denoiser $\Wmat$ from $\Cmat_1,...,\Cmat_K,\bbeta, $ and $\tau/\rho$\;
            $\widehat{\Zmat} \leftarrow \mbox{Algorithm 1}(\Ymat_h, \Ymat_m, \Rmat, \Bmat,\Emat,\Wmat,, \rho, \lambda, \tau)$
			\caption{Hyperspectral Sharpening Algorithm}
		\end{algorithm}
		
		\section{Application to Deblurring Image Pairs}\label{sec:deb}
		\subsection{Formulation}
		Another application that can leverage scene-adapted priors is image deblurring from blurred and noisy image pairs. It was first proposed in \cite{lim}, and it is particularly useful for improving image quality under dim lighting conditions. On the one hand, a long exposure with a hand-held camera is likely to yield blurred pictures, whereas, on the other hand, a short exposure with high ISO produces sharp, yet noisy images. The idea is then to combine these two types of images of the same scene and fuse them to obtain a sharp and clean image.
		
		The observation model in this problem is a simple instance of  \eqref{eq:yh}--\eqref{eq:ym}, with $n_m = n_h = n$, $L_m = L_h = L$, $\Mmat = \Imat$, and $\Rmat = \Imat$, \textit{i.e.}, there is neither sub-sampling nor spectral degradation from the sensor. Furthermore, we only consider grayscale images, thus $L = 1$, and we vectorize the images differently, with all the pixels stacked into a column vector. In conclusion, the observation model is
		\begin{eqnarray}
		\by_b & = & \Bmat \bx + \bn_b,\label{eq:yblur} \\
		\by_n & = & \bx + \bn_n, \label{eq:ynoisy}
		\end{eqnarray}
		where:
		\begin{itemize} 
			\item $\bx \in \mathbb{R}^{n}$ is the (vectorized) target image to be inferred;
			\item $\by_b \in \mathbb{R}^{n}$ is the  blurred image;
			\item $\by_n \in \mathbb{R}^{n}$ is the noisy sharp image;
			\item $\Bmat \in \mathbb{R}^{n \times n}$ represents a spatial blur;
			\item $\bn_b, \bn_n \in \mathbb{R}^{n}$ represent noise, assumed to be zero-mean, white, Gaussian, with variances $\sigma^2_b$ and $\sigma^2_n$, with $\sigma^2_b \ll \sigma^2_n$.
		\end{itemize}
		
	As before, we seek a MAP/variational estimate given by 
				\begin{equation}
		\widehat{\bx}  \in  \underset{\bx}{\argmin}   \;  \frac{\Vert \Bmat \bx  - \by_b \Vert_2^2}{2} +  \frac{ \lambda \Vert \bx  - \by_n \Vert_2^2}{2} + \tau \;\phi(\bx),\label{eq:optimizationpair}
		\end{equation}
		where $\phi$ is the regularizer, and $\lambda$ and $\tau$ have similar meanings as in \eqref{eq:optimization}.

		\subsection{ADMM for Deblurring with Image Pairs} \label{ssec:debpair}
		For this problem formulation, we may use the canonical form of ADMM  in \eqref{eq:admmcan}. Each iteration of the resulting algorithm becomes
		\begin{align}
		\bx^{(k+1)} & =  \arg\min_\bx  \Vert \Bmat \bx  - \by_b \Vert_2^2 +   \lambda \Vert \bx  - \by_n \Vert_2^2 \nonumber \\
		& + \rho \bigl\| \bx -\bv^{(k)} - \bu^{(k)}\bigr\|_2^2   \label{eq:xupdate}\\
		\bv^{(k+1)} & =  \arg\min_\bv  \; \phi(\bv) + \frac{\rho}{2\tau} \bigl\|  \bx^{(k+1)} - \bv -  \bu^{(k)}\bigr\|_2^2 , \label{eq:vupdate}\\
		\bu^{(k+1)} & =  \bu^{(k)} - \bx^{(k+1)} + \bv^{(k+1)}, \label{eq:uupdate}
		\end{align}
		where the objective in \eqref{eq:xupdate} is quadratic, with a closed-form solution,
		\begin{align}
		\bx^{(k+1)} & = \left[ \Bmat^T \Bmat  \!+ \! \lambda\Imat \!+ \!\rho \Imat  \right]^{-1} \left[ \Bmat^T \by_b \!+\! \lambda \by_n\! + \!\rho \left(\bv^{(k)} \!+ \! \bu^{(k)} \right) \right]. \label{eq:pairsFFT}
		\end{align}
		As above, assuming periodic or unknown boundaries, \eqref{eq:pairsFFT} can be implemented with $O(n\log n)$ cost via the FFT \cite{Almeida2013}. The proximity operator in \eqref{eq:vupdate} corresponds to the denoising step, which thus takes the form
		\begin{equation}
		\bv^{(k+1)} =  \Wmat \bigl( \bx^{(k+1)}  -  \bu^{(k)}\bigr). \label{eq:vupdate2}
		\end{equation}
        The proof of convergence of the algorithm in \eqref{eq:xupdate}--\eqref{eq:uupdate} is essentially the same as that of the previous section, so we abstain from presenting it here; we simply have to notice that \eqref{eq:vupdate2} is indeed a proximity operator and that \eqref{eq:optimizationpair} is guaranteed to have at least one solution, because it is a proper, closed, convex function.

       As in the HS sharpening case, $\Wmat$ is not explicitly formed; multiplication by $\Wmat$ corresponds to the patch-based denoiser described in Subsection \ref{sec:GMMdenoiser}, based on the covariance matrices $\Cmat_1,...,\Cmat_K$, and  weights $\beta_j^i$. These covariances are learned via EM from the patches of the noisy sharp image, and the weights of the GMM components at each patch are kept and used to denoise the corresponding patch in the blurred image.

		\section{Experimental Results}
		\label{sec:results}
		\subsection{Hyperspectral Sharpening}
		We begin by reporting an experiment designed to assess the adequacy of the GMM prior. We generate synthetic HS and MS data according to models \eqref{eq:yh} and \eqref{eq:ym}, respectively, with 50dB SNR (nearly noiseless). Afterwards, we add Gaussian noise with a given variance, namely $\sigma = 25$, and recover the clean data using the GMM-based denoiser and BM3D \cite{dabov}. The results are then compared with respect to the visual quality and \textit{peak SNR} (PSNR). Whereas BM3D estimates the clean image from the noisy one, the GMM is trained only once on the nearly noiseless PAN image. This prior is then used on all the denoising experiments: noisy PAN, noisy MS band (red channel), noisy HS band, noisy image of coefficients, see Figs~\ref{fig:den1} to \ref{fig:den4}. The weights $\beta$ are also kept fixed on the four experiments, that is, we use the weights computed in the training stage, relative to the PAN image. In all the experiments, we consider a GMM with 20 components, trained from overlapping $8 \times 8$ patches.
		
		\begin{figure}
			\begin{center}
				\begin{minipage}{.25\linewidth}
					\centering
					\subfloat[]{\includegraphics[width=\textwidth]{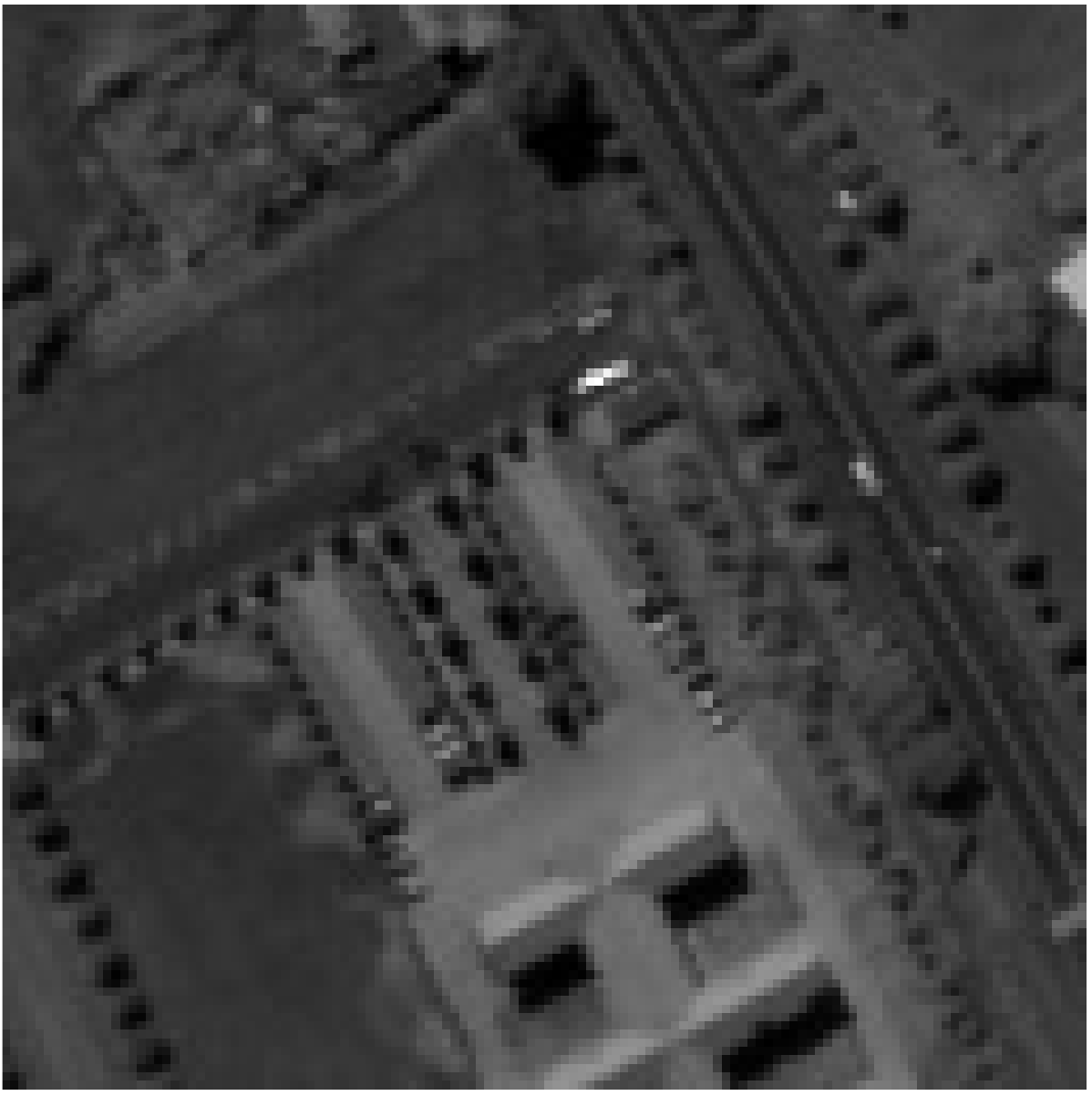}}
				\end{minipage}%
				\begin{minipage}{.25\linewidth}
					\centering
					\subfloat[]{\includegraphics[width=\textwidth]{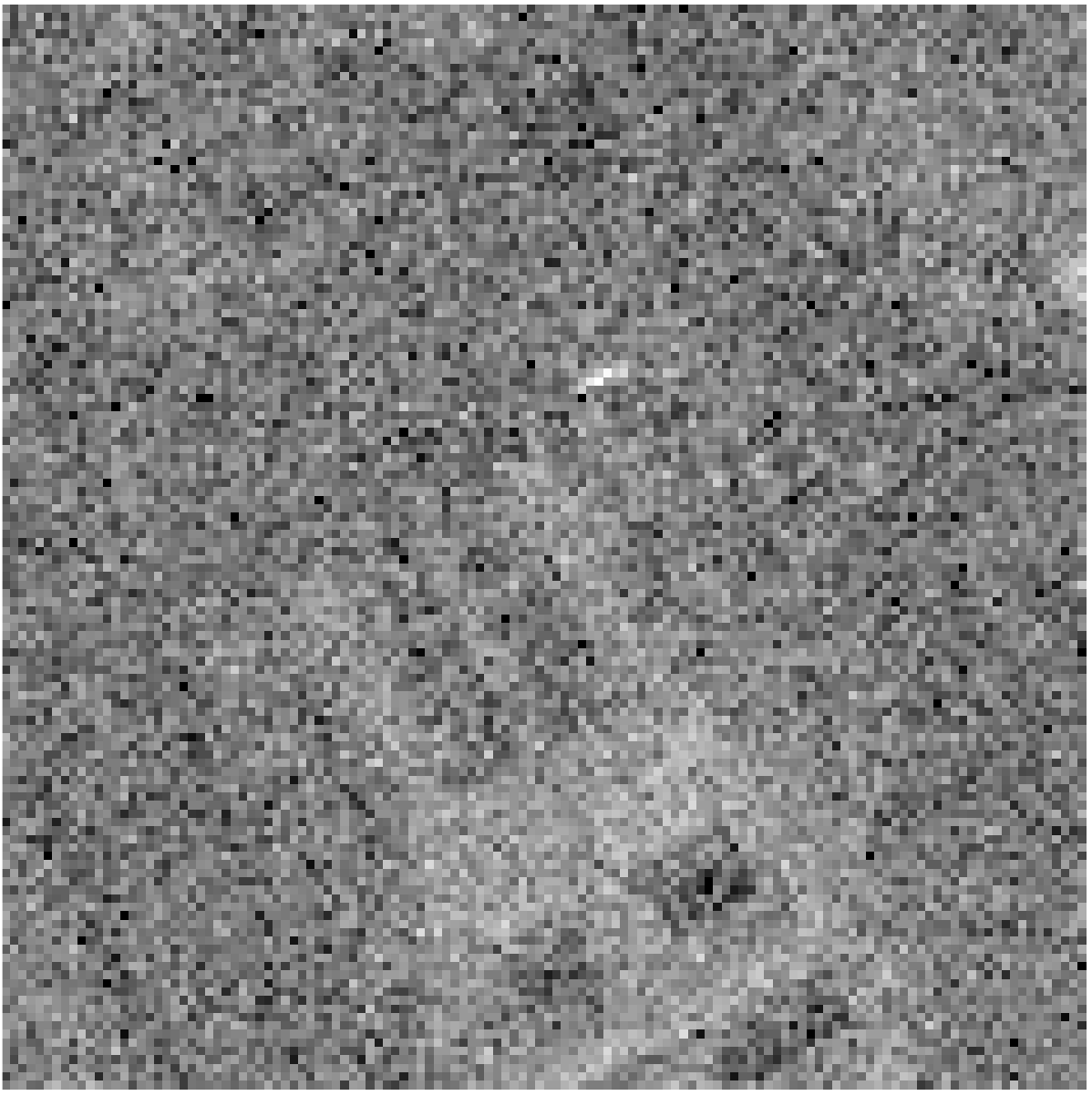}}
				\end{minipage}%
				\begin{minipage}{.25\linewidth}
					\centering
					\subfloat[]{\includegraphics[width=\textwidth]{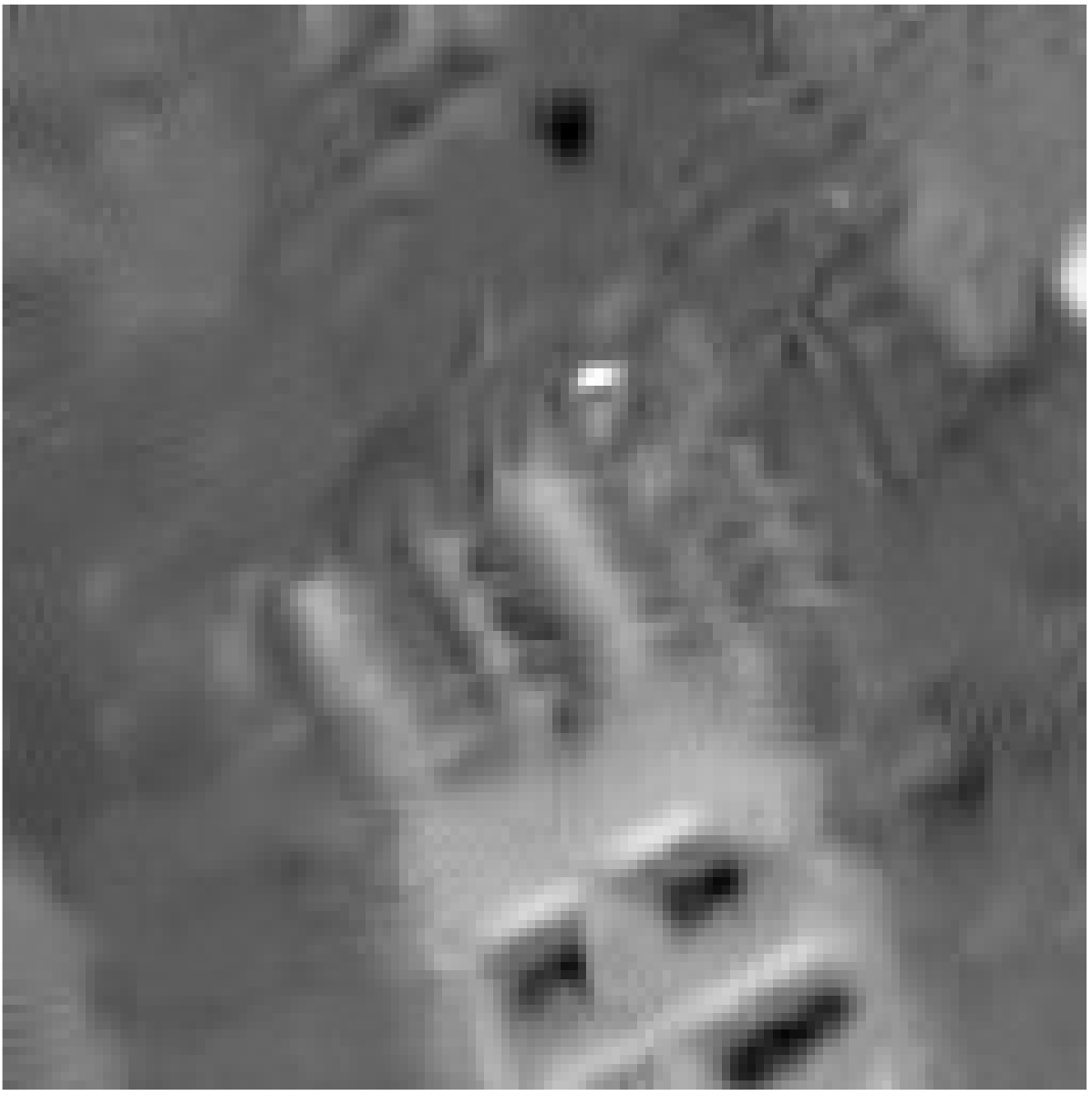}}
				\end{minipage}%
				\begin{minipage}{.25\linewidth}
					\centering
					\subfloat[]{\includegraphics[width=\textwidth]{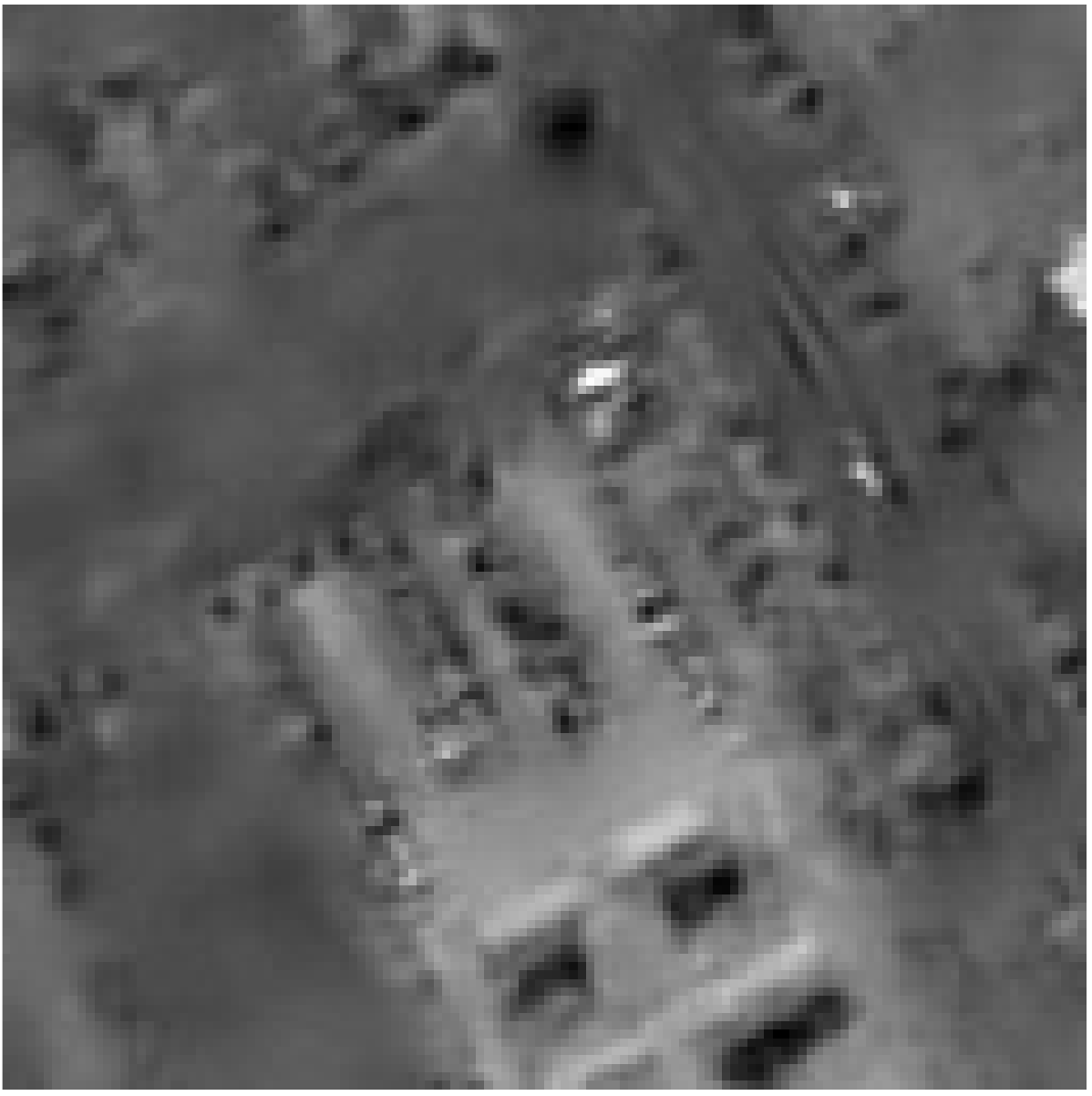}}
				\end{minipage}
			\end{center}
			\vspace{-0.1cm}
			\caption{Denoising: (a) original PAN; (b) noisy PAN ($\sigma = 25$); (c) BM3D (31.90dB); (d) GMM (32.40dB). }
			\label{fig:den1}
		\end{figure}
		
		\begin{figure}
			\begin{center}
				\begin{minipage}{.25\linewidth}
					\centering
					\subfloat[]{\includegraphics[width=\textwidth]{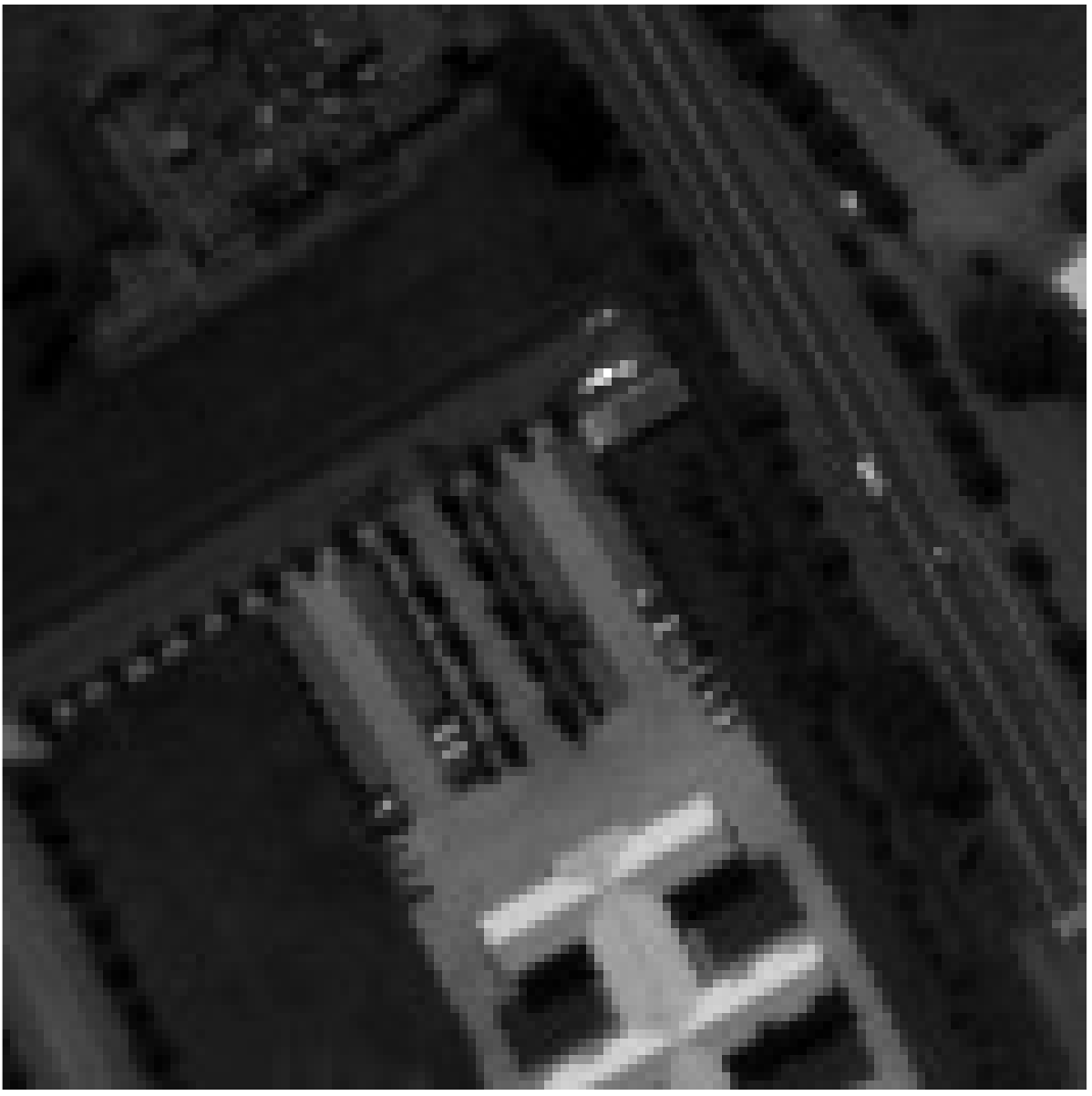}}
				\end{minipage}%
				\begin{minipage}{.25\linewidth}
					\centering
					\subfloat[]{\includegraphics[width=\textwidth]{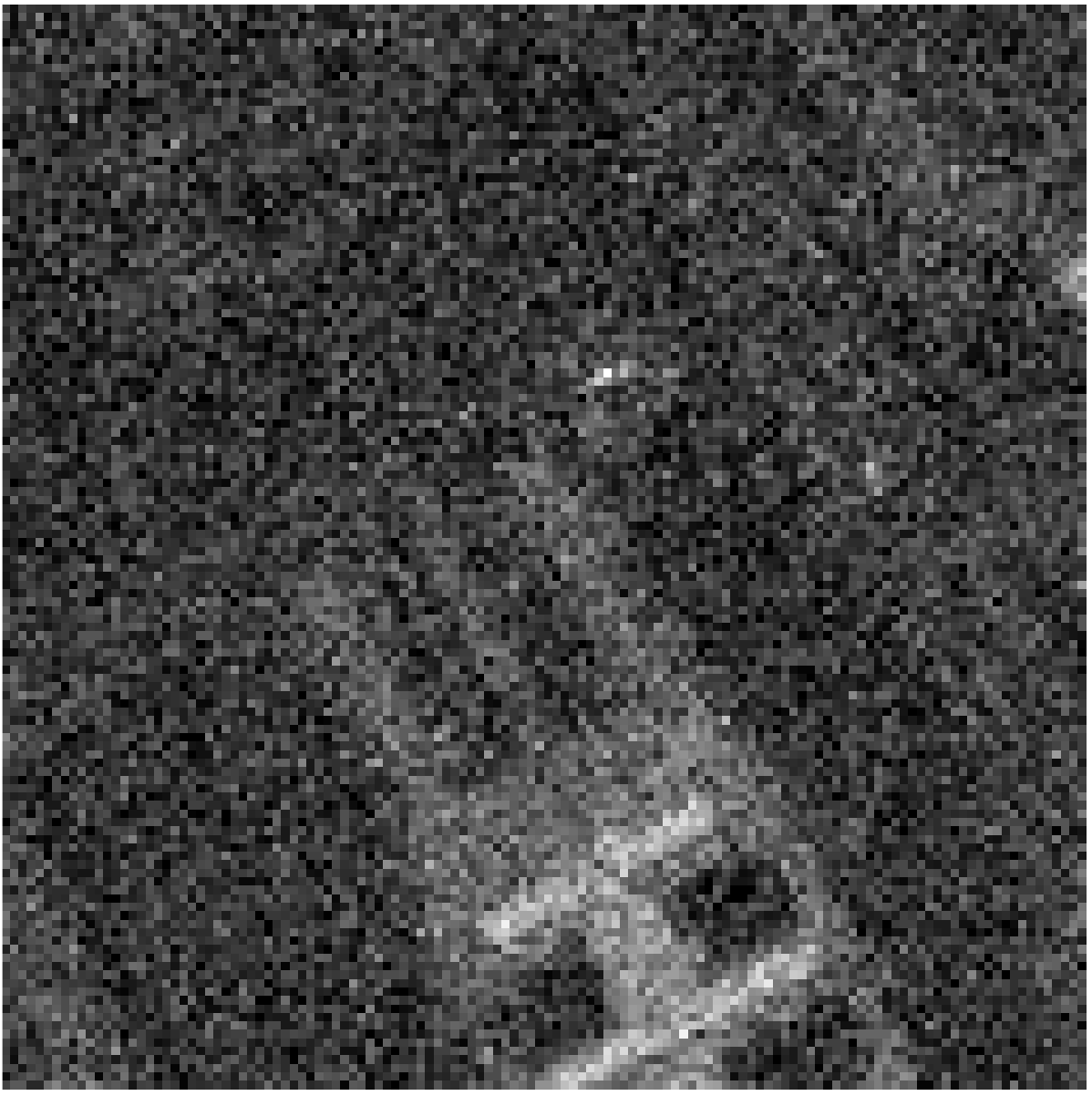}}
				\end{minipage}%
				\begin{minipage}{.25\linewidth}
					\centering
					\subfloat[]{\includegraphics[width=\textwidth]{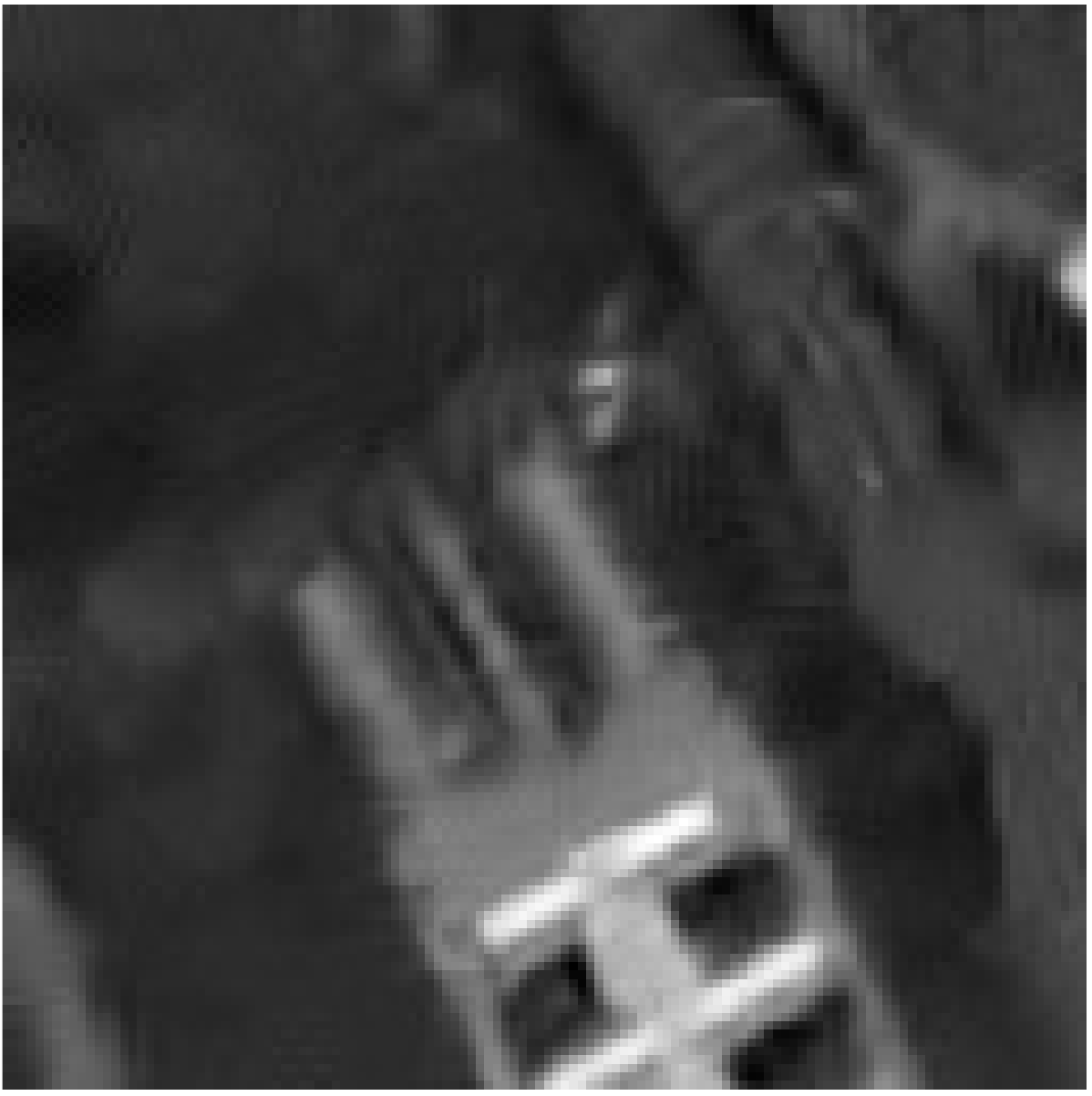}}
				\end{minipage}%
				\begin{minipage}{.25\linewidth}
					\centering
					\subfloat[]{\includegraphics[width=\textwidth]{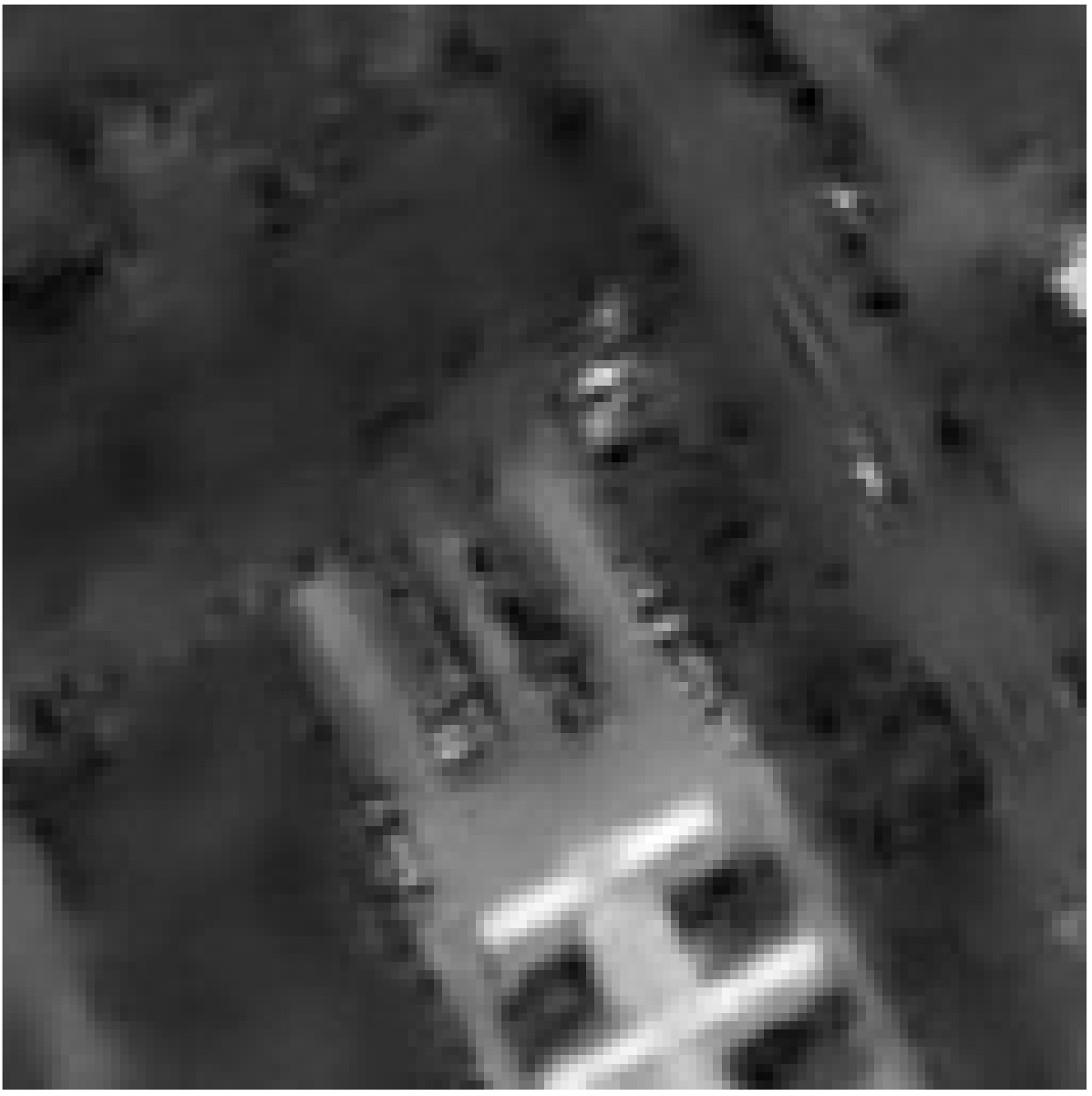}}
				\end{minipage}
			\end{center}
			\vspace{-0.1cm}
			\caption{Denoising: (a) original red band; (b) noisy red band ($\sigma = 25$); (c) BM3D (32.43dB); (d) GMM (32.52dB). }
			\label{fig:den2}
		\end{figure}
		
		\begin{figure}
			\begin{center}
				\begin{minipage}{.25\linewidth}
					\centering
					\subfloat[]{\includegraphics[width=\textwidth]{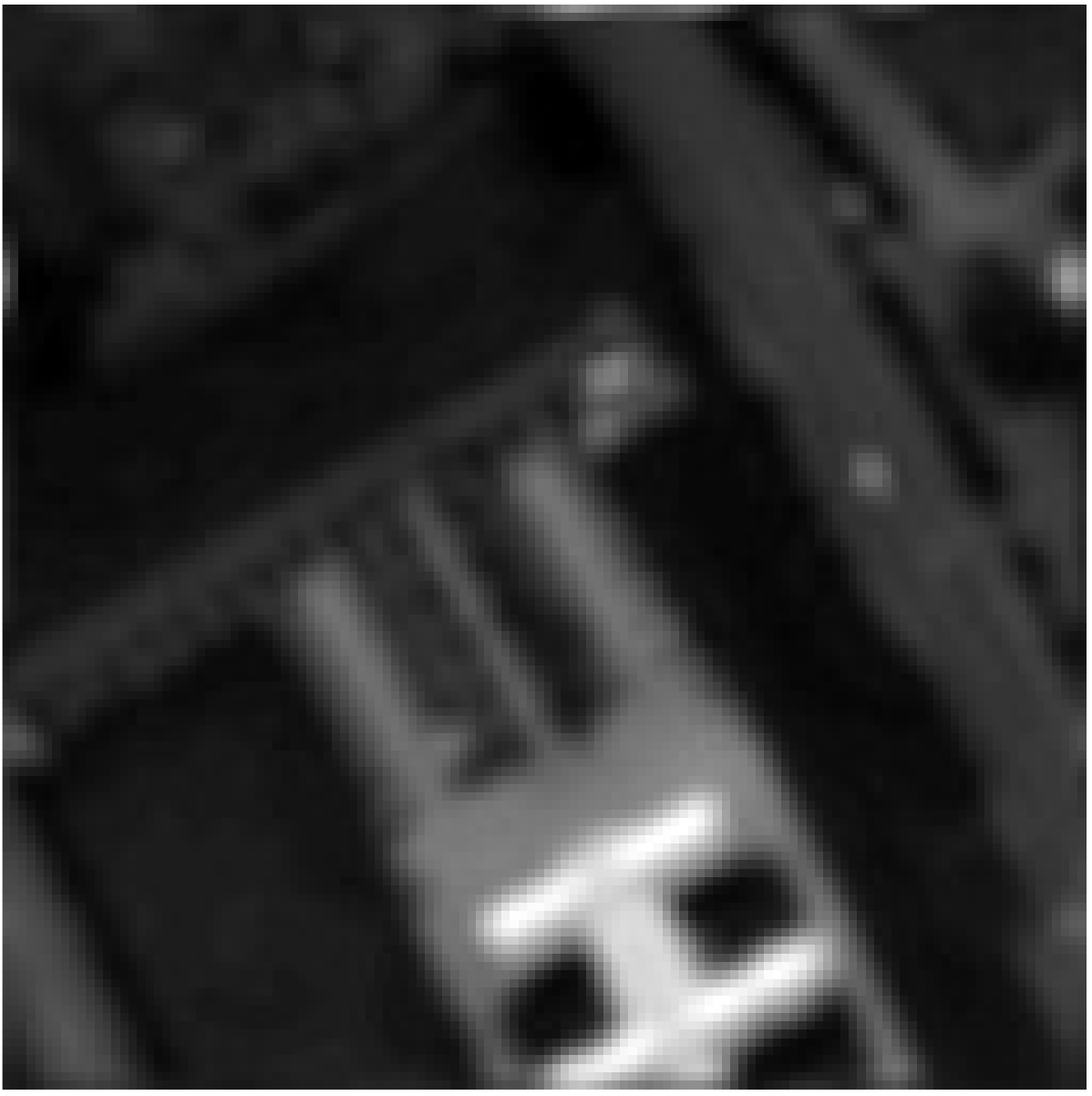}}
				\end{minipage}%
				\begin{minipage}{.25\linewidth}
					\centering
					\subfloat[]{\includegraphics[width=\textwidth]{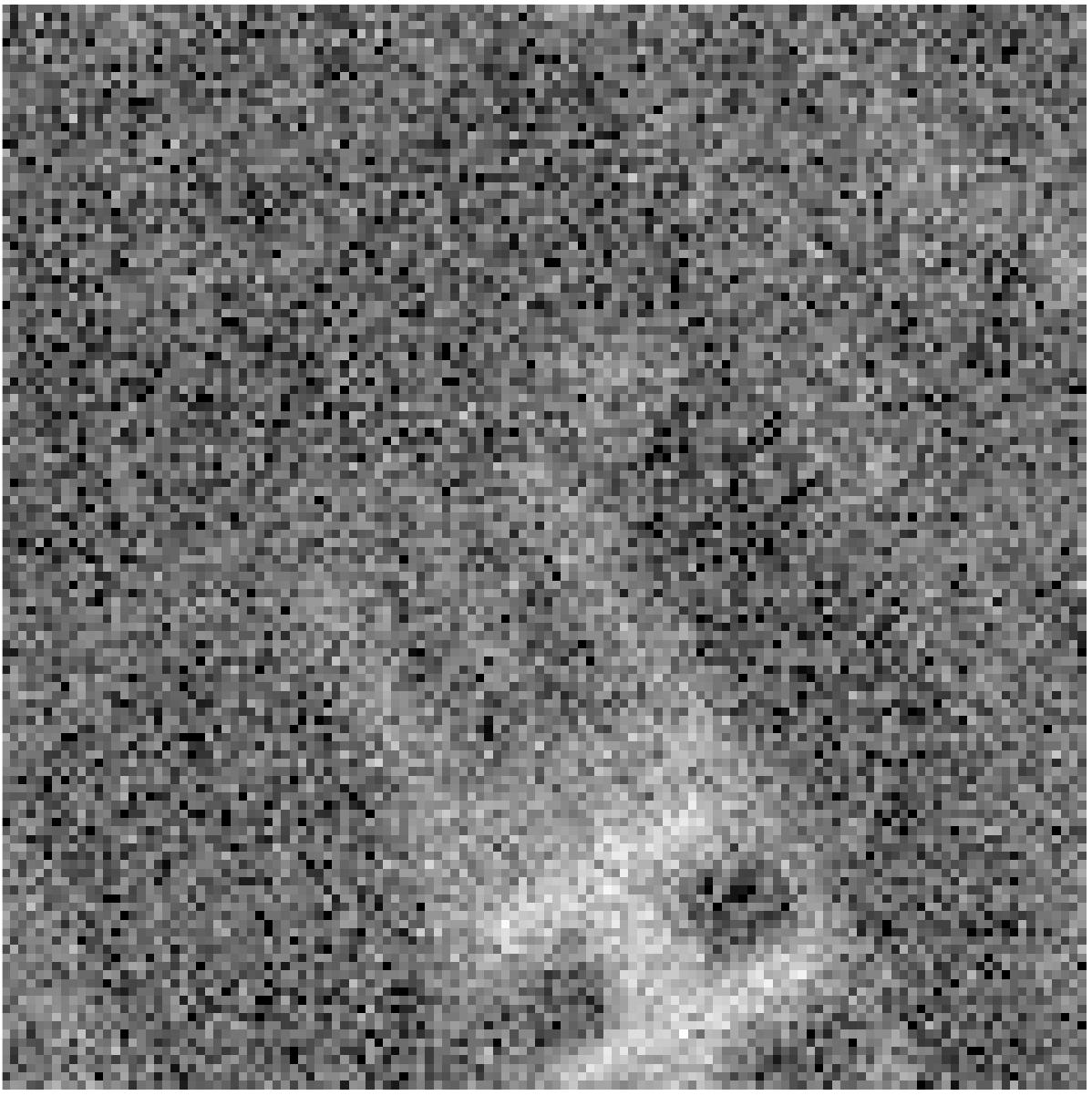}}
				\end{minipage}%
				\begin{minipage}{.25\linewidth}
					\centering
					\subfloat[]{\includegraphics[width=\textwidth]{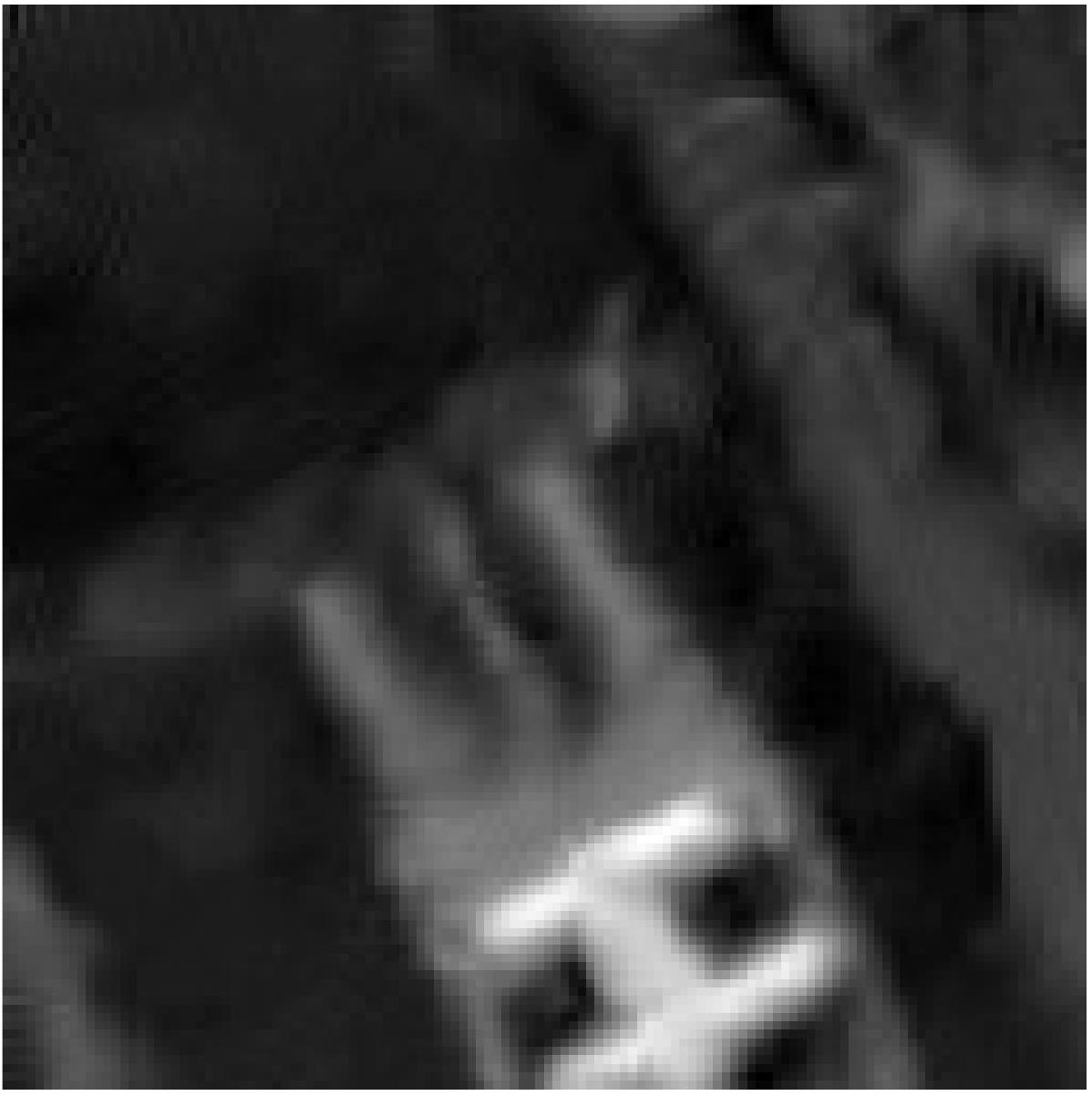}}
				\end{minipage}%
				\begin{minipage}{.25\linewidth}
					\centering
					\subfloat[]{\includegraphics[width=\textwidth]{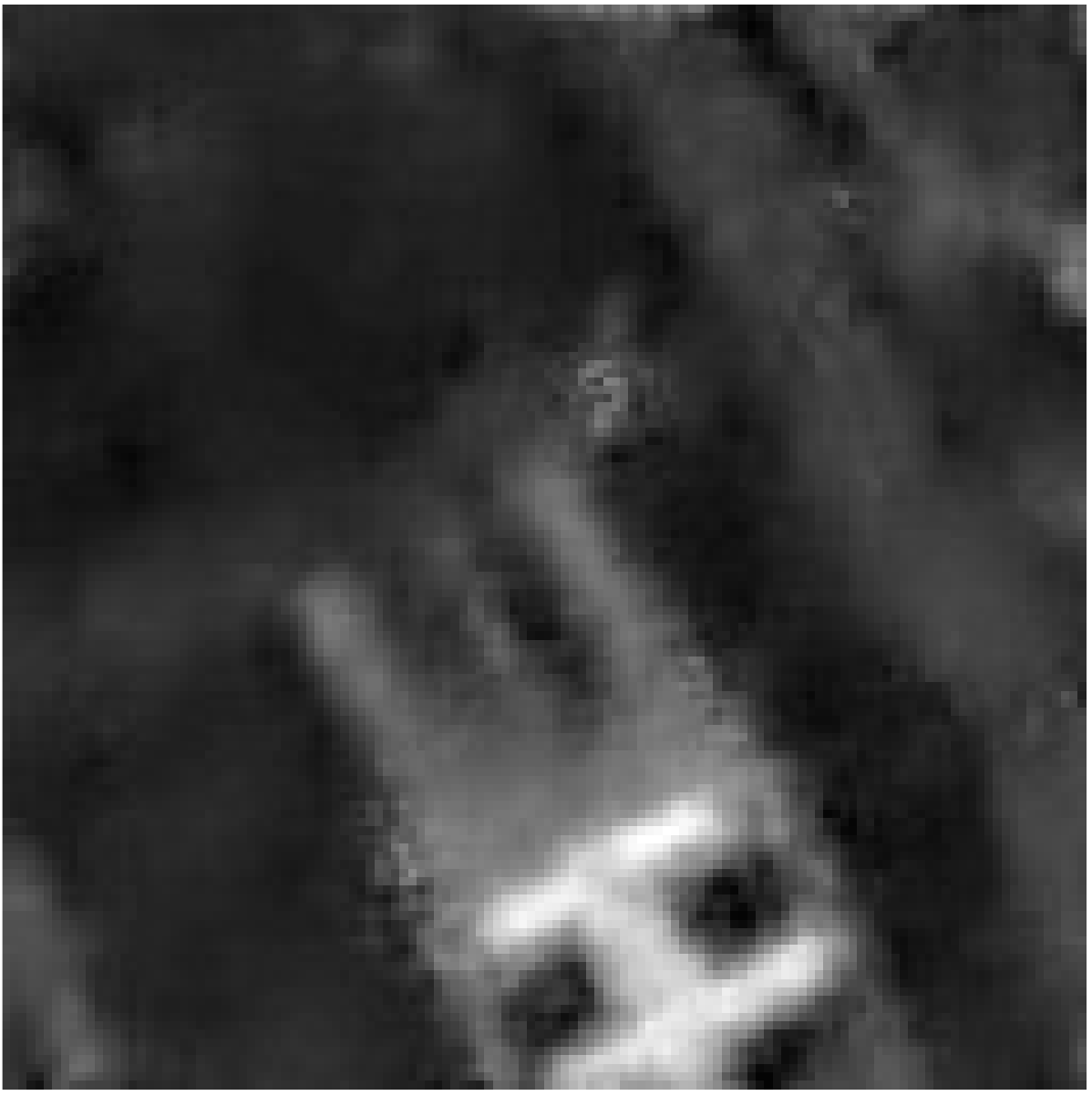}}
				\end{minipage}
			\end{center}
			\vspace{-0.1cm}
			\caption{Denoising: (a) original (blurred) HS band; (b) noisy HS band ($\sigma = 25$); (c) BM3D (36.03dB); (d) GMM (35.90dB). }
			\label{fig:den3}
		\end{figure}
		
		\begin{figure}
			\begin{center}
				\begin{minipage}{.25\linewidth}
					\centering
					\subfloat[]{\includegraphics[width=\textwidth]{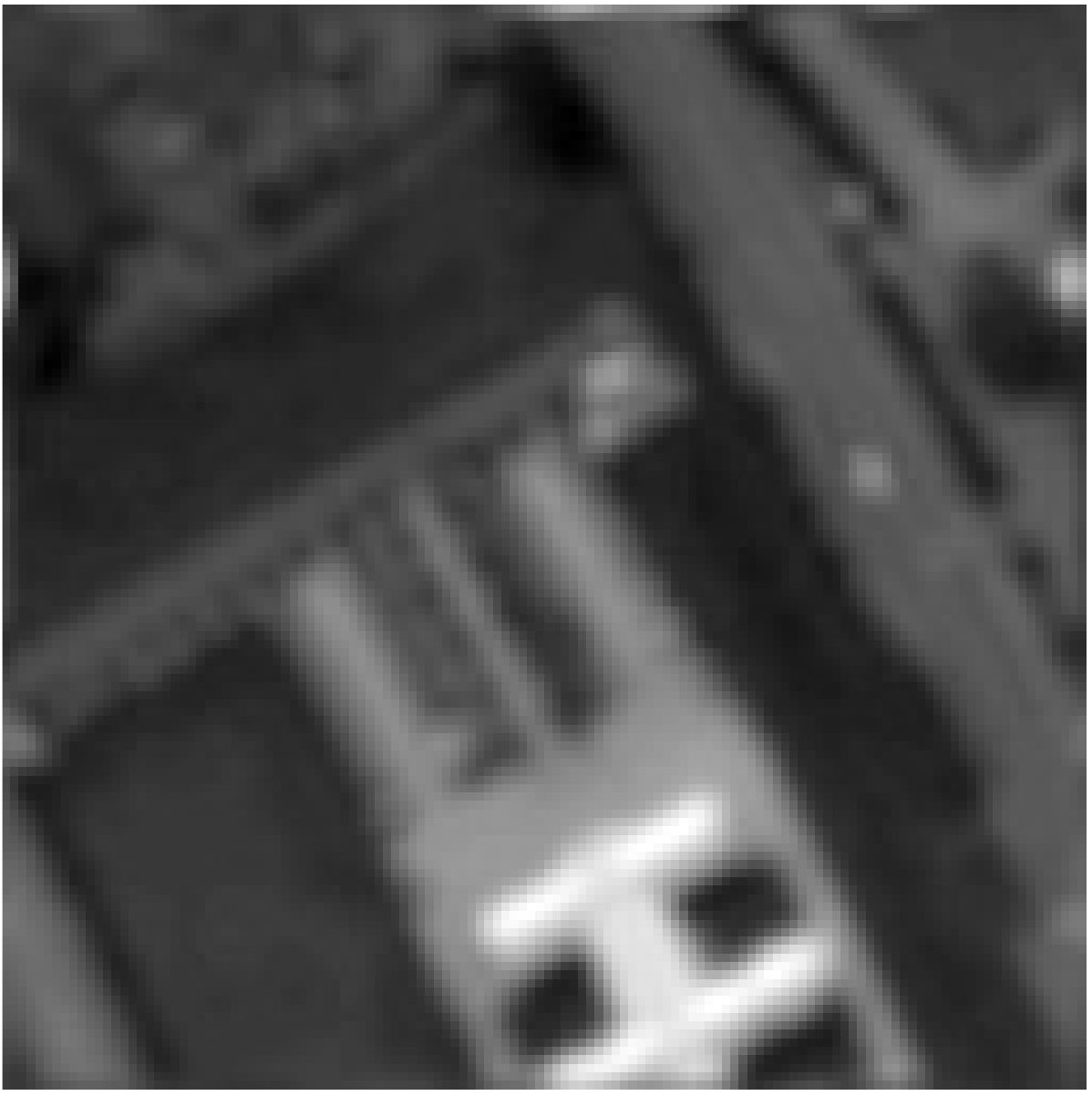}}
				\end{minipage}%
				\begin{minipage}{.25\linewidth}
					\centering
					\subfloat[]{\includegraphics[width=\textwidth]{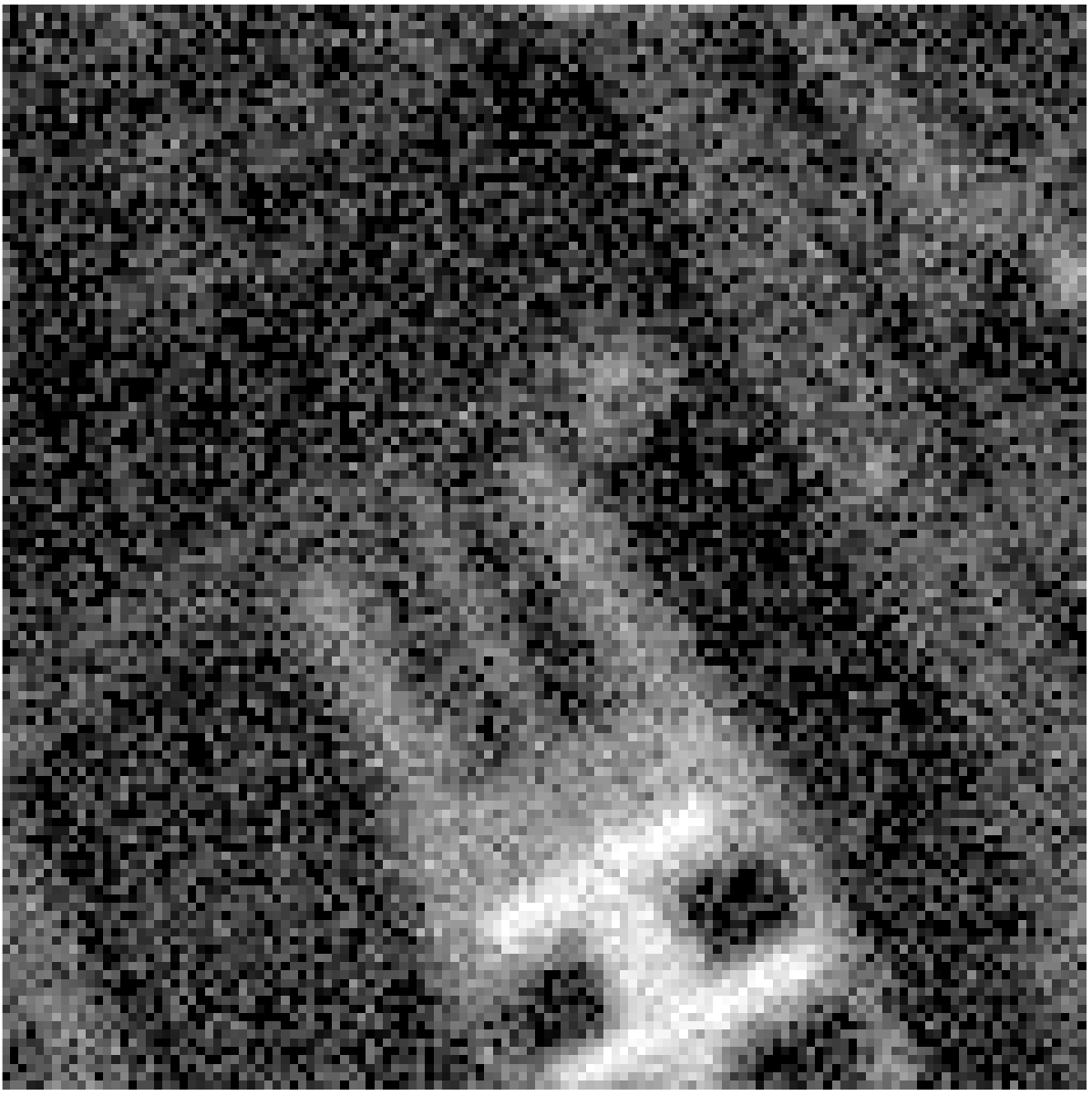}}
				\end{minipage}%
				\begin{minipage}{.25\linewidth}
					\centering
					\subfloat[]{\includegraphics[width=\textwidth]{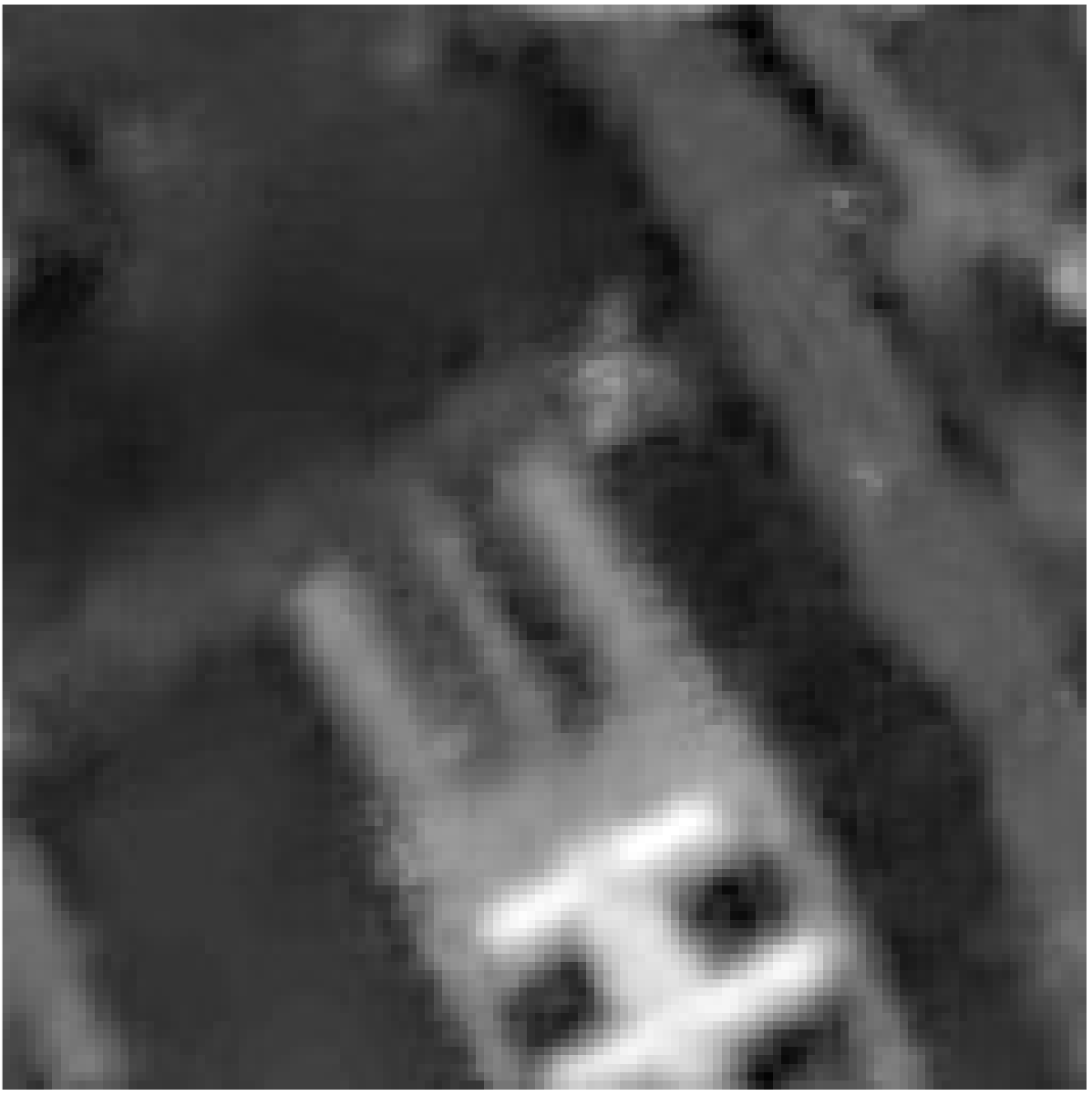}}
				\end{minipage}%
				\begin{minipage}{.25\linewidth}
					\centering
					\subfloat[]{\includegraphics[width=\textwidth]{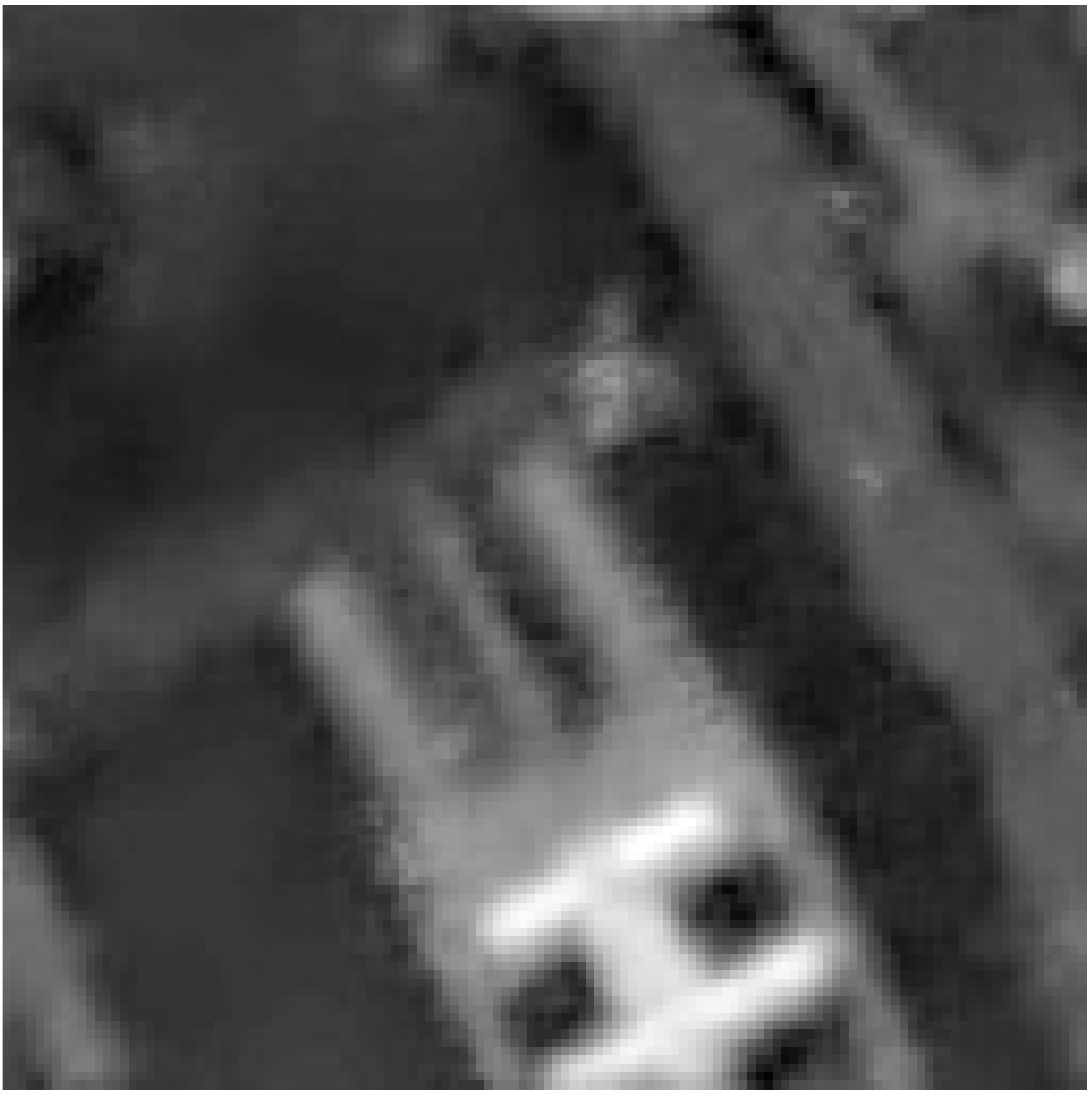}}
				\end{minipage}
			\end{center}
			\vspace{-0.1cm}
			\caption{Denoising: (a) clean coefficients image; (b) noisy coefficients image ($\sigma = 25$); (c) BM3D (33.95dB); (d) GMM (33.62dB). }
			\label{fig:den4}
		\end{figure}

		The  results support the claim that the GMM trained from the PAN image is a good prior for both the MS and HS bands, and also for the image of coefficients. In fact, when denoising the PAN and MS data, the GMM-based denoiser, with the precomputed model and weights, is able to outperform BM3D (Figs~\ref{fig:den1} and \ref{fig:den2}). Furthermore, when denoising the HS image, the same GMM prior is still able to perform competitively with BM3D, which uses only the noisy input (Figs~\ref{fig:den3} and \ref{fig:den4}). While BM3D performs very well in this pure denoising setting, being able to learn the model from the PAN image is an important feature of the GMM denoiser because, in HS sharpening, the spatial resolution of the HS bands may not be enough to learn a suitable prior. On the other hand, since the PAN image has the highest spatial resolution, it is the best band to train the model. Notice that the other MS bands, and also the HS bands and coefficients, retain most of the spatial features of the panchromatic band, which gives supporting evidence to our claim that the weights $\beta$ should be kept fixed. Table~\ref{tab:beta} shows the results when the weights are fixed versus weights computed with \eqref{eq:beta} depending on the noisy input patch (the exact MMSE estimate).
		
		\begin{table}
			\caption{Denoising: fixed weights vs varying weights.\label{tab:beta}}
			\begin{center}
				\resizebox{0.6\columnwidth}{!}{
					\begin{tabular}{c||c||c}

						Image & Fixed $\beta$ & Varying $\beta(\by_i)$ \\ \hline \hline
						PAN       	& \bf 32.40 & 31.97 \\ \hline
						Red 		& \bf 32.52 & 32.31 \\ \hline
						HS       	& \bf 35.90 & 35.86 \\ \hline
						Coefficient & \bf 33.62 & 33.45 \\ \hline \hline
					\end{tabular}
				}
			\end{center}
		\end{table}

		\begin{table*}[htb] 
			\begin{center}
				\caption{HS and MS fusion on (cropped) ROSIS Pavia University and Moffett Field datasets. \label{tab:sharp1}}
				\vspace{3pt}
				\resizebox{0.95\textwidth}{!}{
					\begin{tabular}{c|c||c|c|c||c|c|c||c|c|c||c|c|c}
						\multicolumn{2}{c||}{}	& \multicolumn{3}{c||}{Exp. 1 (PAN)} & \multicolumn{3}{c||}{Exp. 2 (PAN)}  & \multicolumn{3}{c||}{Exp. 3 (R,G,B,N-IR)} & \multicolumn{3}{c}{Exp. 4 (R,G,B,N-IR)} \\
						\hline 
						\hline
						Dataset & Metric & ERGAS & SAM & PSNR & ERGAS & SAM & PSNR & ERGAS & SAM & PSNR & ERGAS & SAM & PSNR\\ \hline \hline
						\multirow{3}{*}{Rosis} & Dictionary \cite{qwei} & 1.97 & 3.25 & 32.80 & 2.04 & 3.14 & 32.22 & \textbf{0.47} & \textbf{0.84} & 45.66 & 0.87 & 1.57 & 39.46 \\ \cline{2-14}
						& MMSE-GMM \cite{Teodoro2017} & 1.62 & \textit{2.66} & 34.09 & 1.72 & 2.69 & 33.46 & \textit{0.49} & \textit{0.87} & 45.64 & \textit{0.84} & 1.38 & 39.94 \\ \cline{2-14}
						& MAP-GMM \cite{Sulam} & \textit{1.56} & 2.56 & \textit{34.49} & \textit{1.64} & \textit{2.59} & \textit{33.94} & \textit{0.49} & \textit{0.87} & \textit{45.67} & \textbf{0.80} & \textit{1.35} & \textbf{40.34} \\ \cline{2-14}
						& \textbf{SA-GMM} & \textbf{1.55} & \textbf{2.53} & \textbf{34.57} & \textbf{1.64} & \textbf{2.57} & \textbf{33.96} & \textit{0.49} & \textit{0.87} & \textbf{45.68} & \textbf{0.80} & \textbf{1.34} & \textit{40.31} \\ \hline
						\multirow{3}{*}{Moffett} & Dictionary \cite{qwei} & 2.67 & 4.18 & 32.86 & 2.73 & 4.19 & 32.51  & 1.83 & 2.70 & 39.13 & 2.13 & 3.19 & 36.15 \\ \cline{2-14}
						& MMSE-GMM \cite{Teodoro2017} & 2.56 & 4.03 & 32.98 & 2.64 & 4.04 & 32.68 & 1.85 & 2.69 & 39.14 & \textbf{1.98} & 2.97 & 36.50 \\ \cline{2-14}
						&  MAP-GMM \cite{Sulam} & \textbf{2.47} & \textit{3.91} & \textit{33.28} & \textbf{2.55} & \textit{3.91} & \textit{32.96} & \textit{1.69} & \textit{2.52} & \textit{39.57} & \textit{2.01} & 2.88 & 36.61 \\ \cline{2-14}
						& \textbf{SA-GMM} & \textbf{2.47} & \textbf{3.89} & \textbf{33.30} & \textbf{2.55} & \textbf{3.90} & \textbf{32.97} & \textbf{1.67} & \textbf{2.47} & \textbf{39.66} & \textbf{1.98} & \textbf{2.84} & \textbf{36.76} \\ \hline
						\hline
					\end{tabular}
				}
			\end{center}
		\end{table*}
		
		Next, we compare the HS sharpening results using the MMSE and MAP GMM-based denoisers from \cite{Teodoro2015} and \cite{Sulam}, respectively, the scene-adapted GMM denoiser herein proposed, and the dictionary-based method from \cite{qwei} (which, to the best of our knowledge, is state-of-the-art). We use three different metrics: ERGAS ({\em erreur relative globale adimensionnelle de synth\`ese}), SAM ({\em spectral angle mapper}), and  PSNR \cite{review}, in 4 different settings. In the first one, we consider sharpening the HS images using the PAN image, both at 50dB SNR. The second experiment refers to PAN-sharpening as well, at 30dB SNR on the PAN image and last 50 HS bands, and 35dB on all the other bands \cite{qwei}. Experiments 3 and 4 use the same SNR as 1 and 2, respectively, but the HS bands are sharpened based on 4 MS bands: R, G, B, and near-infrared. Once again, a GMM with $K=20$ components is learned, and used as explained in Subsection \ref{sec:complete}. The PAN and MS images were generated from the original HS data using the IKONOS and LANDSAT spectral responses. For more details about the experimental procedure we refer the reader to \cite{qwei}.
		
		Table~\ref{tab:sharp1} shows the results on a cropped area of the \textit{ROSIS Pavia University} and \textit{Moffett Field} datasets. The four methods have comparable performances, with the proposed method being slightly better. The results also support the claim that the target image has the same spatial structure as the PAN image used to train the GMM, otherwise keeping the posterior weights would not yield good results. Figures~\ref{fig:orig}--\ref{fig:rec} show the results of experiment 4 on the Moffett Field data, in terms of visual quality, yet the differences are not visually noticeable. To better grasp the accuracy, we plot the (sorted) errors for each pixel, as well as pixel values across bands. In all experiments, the parameters of PnP-SALSA, as described on Algorithm 1, were selected using a grid search. In particular, in this example, we used $\rho = 10^{-4}$, $\lambda = 10^{-1}$, $\tau = 10^{-6}$. 
	
		\begin{figure}[htb]\centering
			\begin{minipage}{.25\linewidth}
				\centering
				\subfloat[]{\includegraphics[width=\textwidth]{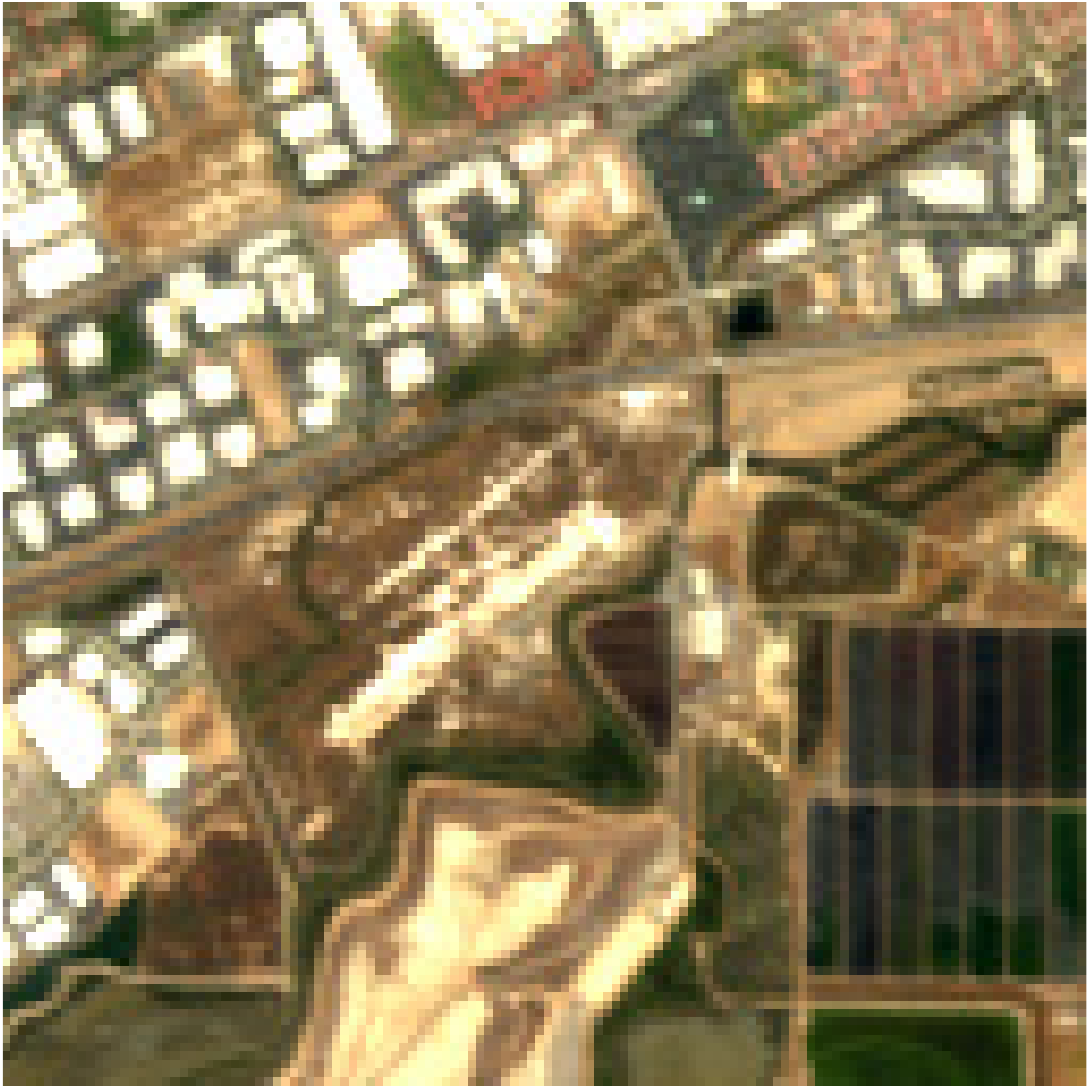}\label{fig:orig}}
			\end{minipage}%
			\begin{minipage}{.25\linewidth}
				\centering
				\subfloat[]{\includegraphics[width=\textwidth]{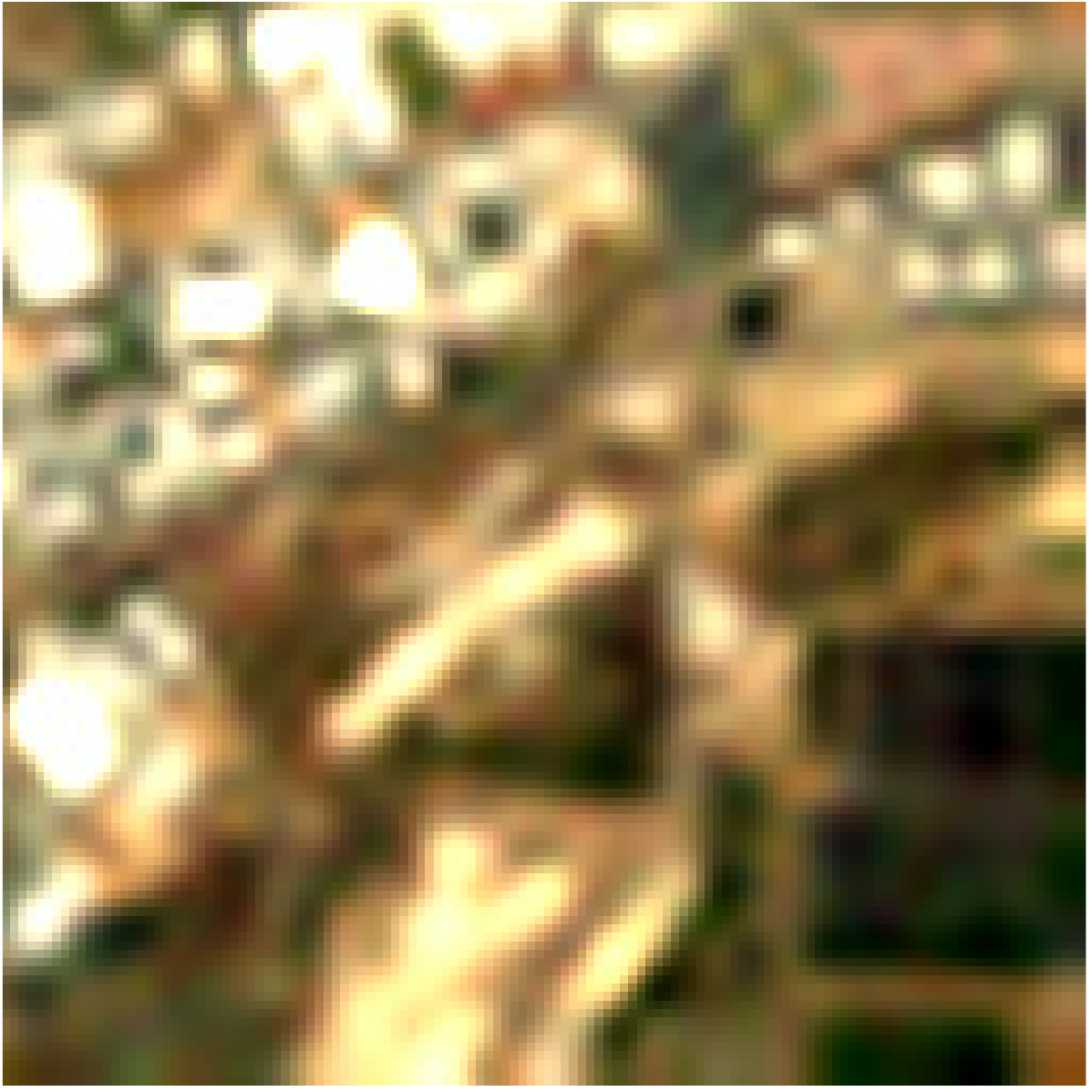}}
			\end{minipage}%
			\begin{minipage}{.25\linewidth}
				\centering
				\subfloat[]{\includegraphics[width=\textwidth]{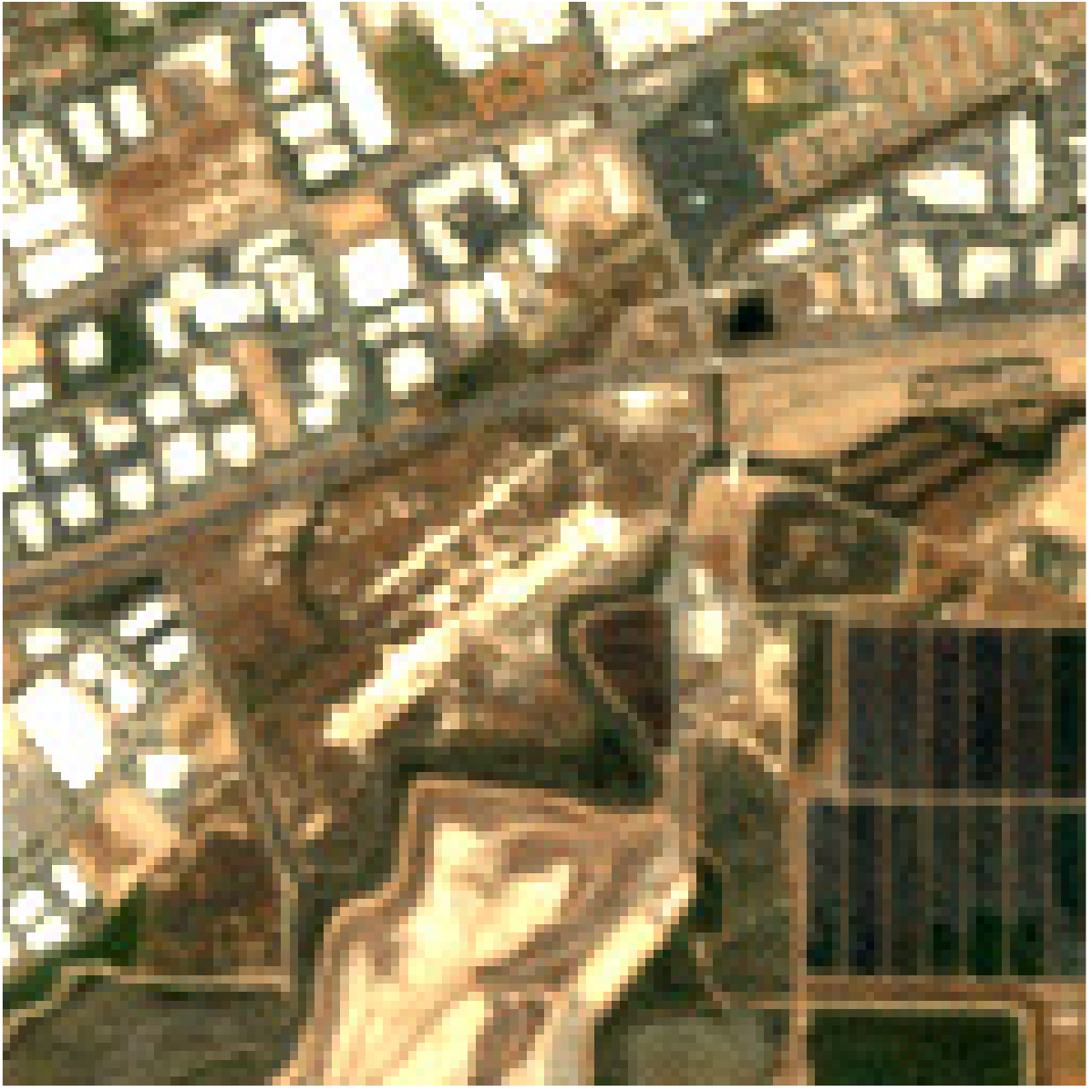}}
			\end{minipage}%
			\begin{minipage}{.25\linewidth}
				\centering
				\subfloat[]{\includegraphics[width=\textwidth]{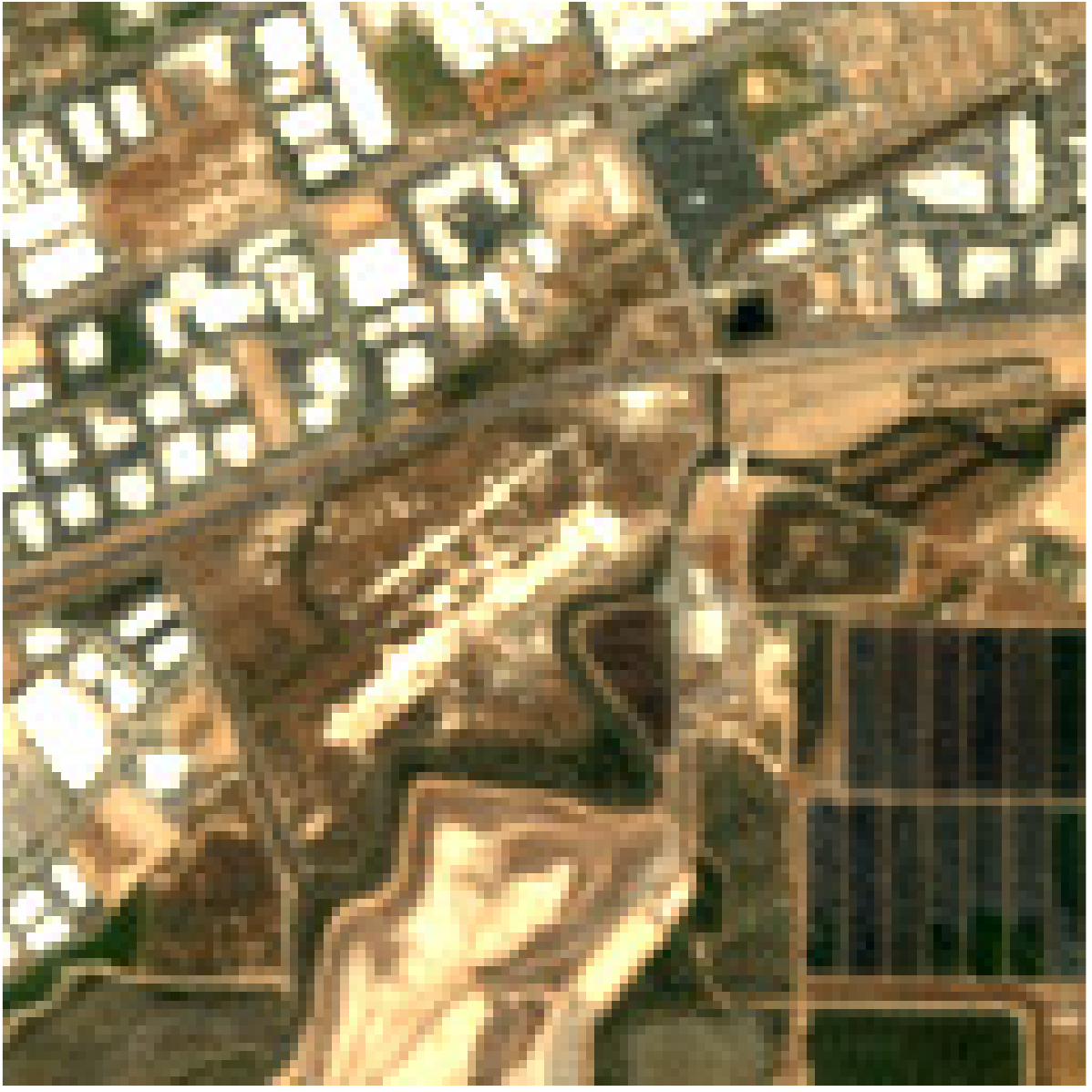}\label{fig:rec}}
			\end{minipage}
			\begin{center}
				\begin{minipage}{.33\linewidth}
					\centering
					\subfloat[]{\includegraphics[width=0.97\textwidth]{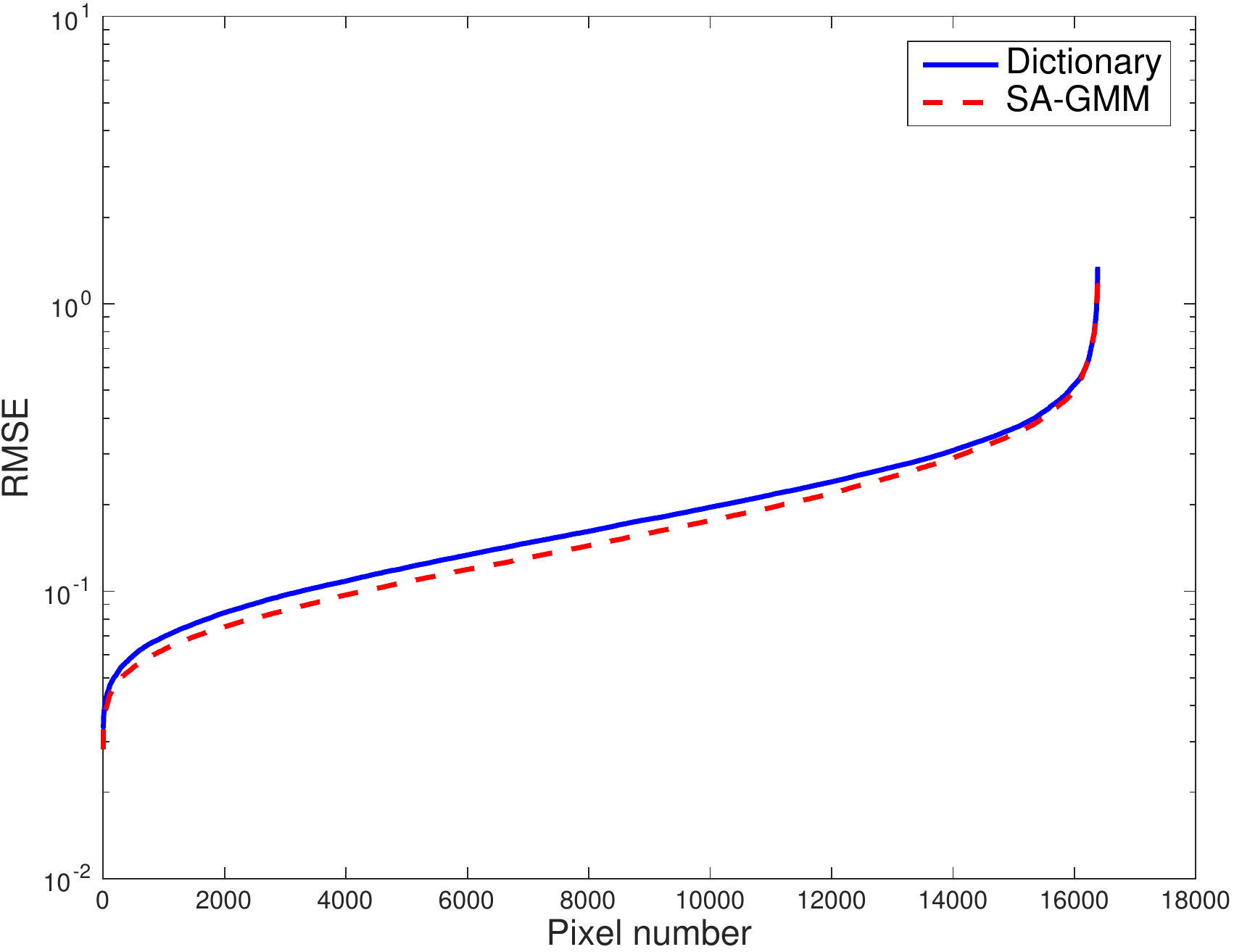}\label{fig:plotnorms2}}
				\end{minipage}%
				\begin{minipage}{.33\linewidth}
					\centering
					\subfloat[]{\includegraphics[width=0.99\textwidth]{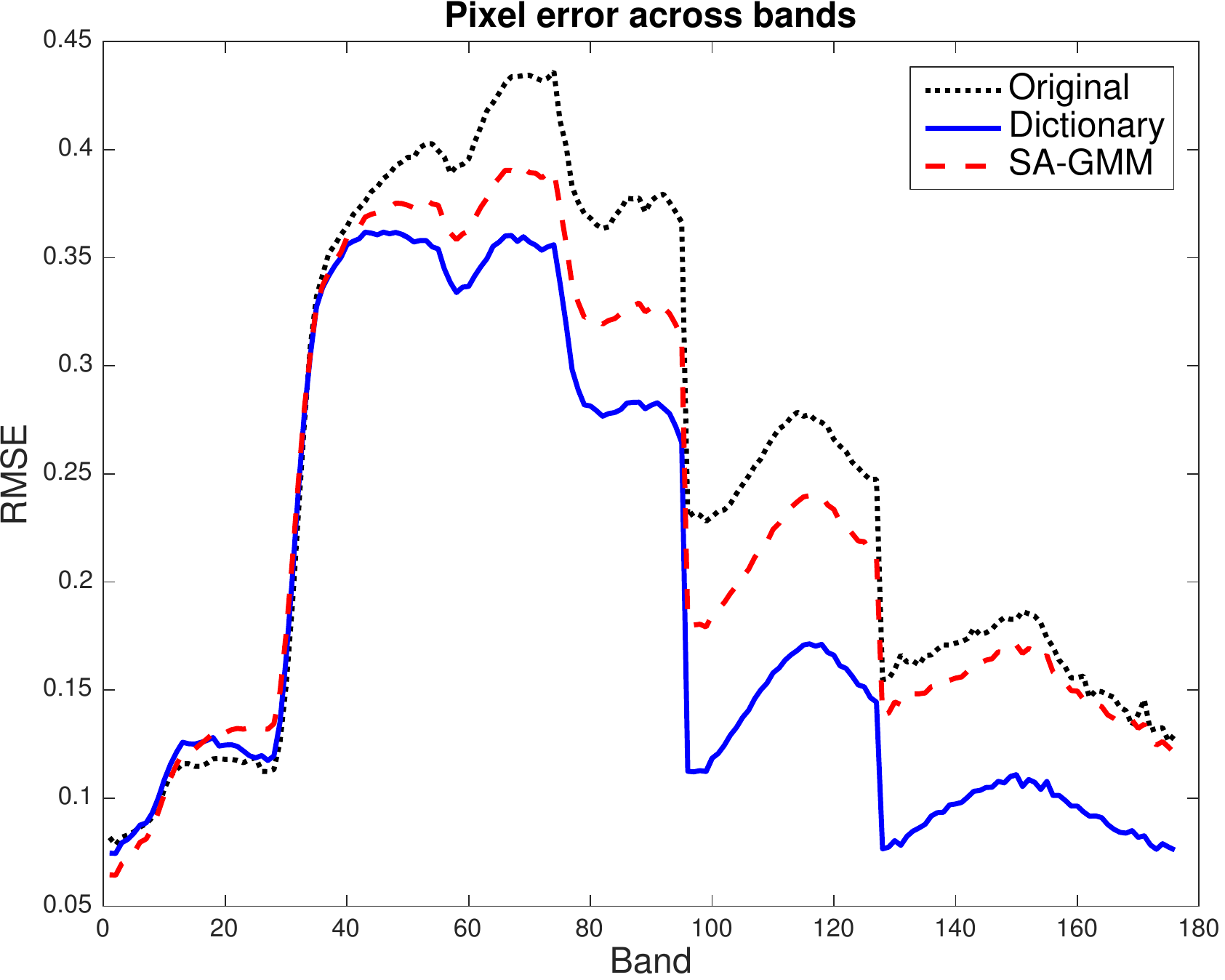}\label{fig:err1}}
				\end{minipage}%
				\begin{minipage}{.33\linewidth}
					\centering
					\subfloat[]{\includegraphics[width=0.99\textwidth]{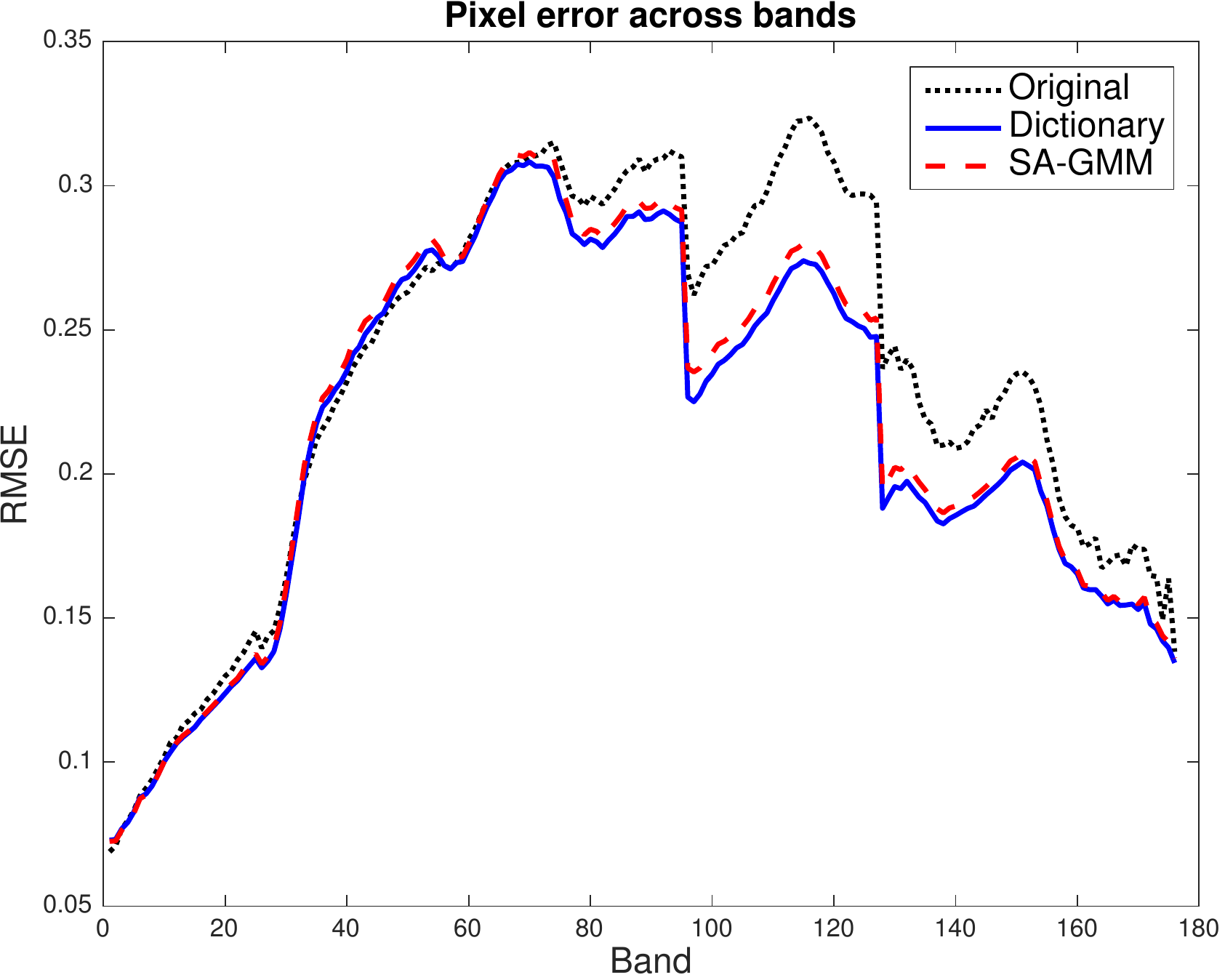}\label{fig:err2}}
				\end{minipage}
			\end{center}
			\vspace{-0.4cm}
			\caption{(a) Original HS bands in false color (20, 11, 4); (b) low-resolution; (c) dictionary-based \cite{qwei}; (d) GMM-based (proposed); (e) sorted pixel errors; (f)--(g) restored pixel value across bands.}
			\label{fig:sharp2}
		\end{figure}
		

		\subsection{Deblurring with Image Pairs}
		In image deblurring from blurred/noisy pairs, we compare our scene-adapted denoiser plugged into ADMM versus BM3D and the state-of-the-art method for non-blind deblurring IDD-BM3D \cite{danielyan}. Our goal is to show that having more data, even if some of it is degraded, leads to better restoration results. Using two classical benchmark images, we applied BM3D for different noise levels and IDD-BM3D for 3 of the blur kernels described in \cite{danielyan} ($\#1,\, \#3,$ and $\#6$). Then for every combination of kernel and noise level (refered to as image pair), we applied the algorithm described in Subsection \ref{ssec:debpair}, which is a simplified version of Algorithm 2. The results in Table~\ref{tab:sharp2} show that, except for the nearly noiseless case, using a pair of images and a fusion approach achieves higher PSNR than denoising only, as would be expected. On the other hand, the fusion approach only outperforms the non-blind deblurring result for low to medium noise levels on the noisy image. This is because higher noise levels lead to poorer models. Also, for strong blur (kernel $\#6$), our fusion approach  only matches the IDD-BM3D method in the nearly noiseless case. Once again, in the fusion experiments, we used GMMs with 20 components, trained from $8 \times 8$ patches extracted from the noisy input image, and parameters $\tau$ and $\rho$ were hand-tuned for best results, over a grid of possible values.

		\begin{table}[htb] 
			\begin{center}
				\caption{Deblurring image pairs; PSNR. \label{tab:sharp2}}
				\vspace{3pt}
				\resizebox{\columnwidth}{!}{
					\begin{tabular}{c|c||c||c|c|c||c|c|c}
						\multirow{2}{*}{Dataset} & \multirow{2}{*}{$\sigma_n$} & \multicolumn{1}{c||}{\multirow{2}{*}{Denoising}} & \multicolumn{3}{c||}{Fusion}  & \multicolumn{3}{c}{Deblurring}\\ \cline{4-9}
						& & & $\#1$ & $\#3$ & $\#6$ & $\#1$ & $\#3$ & $\#6$ \\ \hline \hline
						\multirow{5}{*}{Cman} & 5 & 38.29 & 37.93 & 39.12 & 34.73 & \multirow{5}{*}{31.08} & \multirow{5}{*}{31.21} & \multirow{5}{*}{34.71}\\ \cline{2-6}
						& 15 & 31.91 & 32.41 & 33.72 & 33.62 & & &  \\ \cline{2-6}
						& 25 & 29.45 & 30.72 & 31.70 & 32.85 & & & \\ \cline{2-6}
						& 50 & 26.12 & 29.59 & 30.19 & 31.33 & & &  \\ \cline{2-6}
						& 75 & 24.33 & 28.65 & 28.81 & 30.80 & & & \\ \hline
						\multirow{5}{*}{House} & 5 & 39.83 & 39.59 & 40.92 & 37.93 & \multirow{5}{*}{35.62} & \multirow{5}{*}{37.00} & \multirow{5}{*}{37.11} \\ \cline{2-6}
						& 15 & 34.94 & 35.87 & 37.55 & 36.08 & & & \\ \cline{2-6}
						& 25 & 32.86 & 35.07 & 36.77 & 35.48 & & &  \\ \cline{2-6}
						& 50 & 29.69 & 33.81 & 35.38 & 33.72 & & & \\ \cline{2-6}
						& 75 & 27.51 & 32.80 & 34.33 & 32.26 & & & \\ \hline
						\hline
					\end{tabular}
				}
			\end{center}
		\end{table}

		\section{Conclusion}
		\label{sec:conclusion}

		We proposed a plug-and-play (PnP) ADMM algorithm, with a GMM-based denoiser, and showed that it is guaranteed to converge to a global minimum of the underlying cost function. The denoiser is scene-adapted and we showed that it is the proximity operator of a closed, proper, and convex function, allowing the standard convergence guarantees of ADMM to be invoked. We then proposed two applications of the algorithm to data fusion problems: hyperspectral (HS) sharpening and image deblurring from noisy/blurred image pairs.
		
		Experimental results on HS sharpening showed that the proposed method outperforms, on most of the test settings, another state-of-the-art algorithm based on sparse representations on learned dictionaries \cite{qwei}. Furthermore, the proposed scene-adaptation also consistently improves the results over the GMM-based denoiser in \cite{Teodoro2017},  supporting the rationale behind the use of scene-adapted priors.

		As future work, we will consider extending this approach to blind scenarios, where some of the components of the model are unknown, namely the blur operator $\Bmat$. The presence of observations that do not depend on $\Bmat$ (\textit{e.g.}, the noisy sharp image) is expected to facilitate the estimation of this operator.
		
			\section*{Appendix A: Proof of Theorem~\ref{lem:1}}
		
		\begin{proof} We first prove claim (i). Notice that, since $\Wmat$ is symmetric, $\Wmat \by = \nabla_{\by} (\frac{1}{2}\by^T \Wmat \by)$, which is convex because $\Wmat$ is p.s.d. Thus, using the fact that $\lambda_{\mbox{\scriptsize max}} \leq 1 \Rightarrow \|\Wmat\bx - \Wmat \bx'\|_2 \leq \|\bx-\bx'\|_2$, for any $\bx,\bx'\in\mathbb{R}^n$, and invoking Lemma \ref{lem:moreau} with $\varphi(\by) = \frac{1}{2}\by^T \Wmat \by$ (thus $\partial \varphi(\by) = \{\Wmat \by\}$), shows that $\Wmat\by = \mbox{prox}_{\phi}(\by)$ for some convex function $\phi$.
			
			We now prove claim (ii). Start by observing that $\phi$ is indeed convex, because it is the sum of two convex functions: the indicator of a subspace (which is a convex set) and a convex quadratic function (since $\lambda_{\mbox{\scriptsize max}} \leq 1$, the entries in $\Lambda^{-1}-\Imat$ are all non-negative). Function $\phi$ is closed because it is the sum of two closed functions: the indicator of a closed convex set (finite-dimensional subspaces are closed sets) and a continuous function. To show that $\phi$ is proper, simply notice that $\bx^T \bar{\Qmat} ( \bar{\Lambda}^{-1} - \Imat ) \bar{\Qmat}^T \bx \neq +\infty$, for any $\bx\in\mathbb{R}^n$, and $S(\Wmat)\neq \emptyset$.
			
			Using the definition of $\mbox{prox}_{\phi}$ (see \eqref{eq:prox}), we have 
			\begin{eqnarray}
				\lefteqn{\mbox{prox}_{\phi}(\by)  = } \nonumber\\
				& = & \arg\min_{\bx} \|\bx - \by\|_2^2 + 2\, \iota_{S(\Wmat)} (\bx) +\bx^T \bar{\Qmat} ( \bar{\Lambda}^{-1} - \Imat ) \bar{\Qmat}^T \bx \nonumber\\
				& = &\arg\min_{\bx \in S(\Wmat) }  \|\bx - \by\|_2^2 + \bx^T \bar{\Qmat} ( \bar{\Lambda}^{-1} - \Imat ) \bar{\Qmat}^T \bx .\label{eq:ppp}
			\end{eqnarray}
			Any  $\bx \in S(\Wmat)$ can be written as $\bx = \bar\Qmat\bz$, for some $\bz\in\mathbb{R}^{r}$, thus \eqref{eq:ppp} can be written in unrestricted form as 	
			\begin{eqnarray}
				\lefteqn{\mbox{prox}_{\phi}(\by)  = } \nonumber\\ 
				& = & \bar\Qmat \;\arg\min_{\bz \in \mathbb{R}^{r} } \, \| \bar{\Qmat}\bz - \by\|_2^2 + \bz^T \bar{\Qmat}^T\bar{\Qmat}  ( \bar{\Lambda}^{-1} - \Imat )  \bar{\Qmat}^T\bar{\Qmat} \bz \nonumber\\
				& = & \bar\Qmat \; \arg\min_{\bz \in \mathbb{R}^{r} } \,  \bz^T \bar{\Lambda}^{-1} \bz\  -2\, \bz^T \bar{\Qmat}^T \by\\
				& = & \bar\Qmat \; \bar{\Lambda} \; \bar{\Qmat}^T\, \by
			\end{eqnarray}
			Finally, the fact that $\bar\Qmat \bar{\Lambda}  \bar{\Qmat}^T \by = \Wmat \by$ concludes the proof.
		\end{proof}
	
		\section*{Appendix B : Proof of Lemma~\ref{lem:22}}
		
		\begin{proof}
		Each $\bF_i$ (see \eqref{eq:linMMSE}) is a convex combination of symmetric matrices, thus also symmetric; consequently, $\Wmat$ is also symmetric (see \eqref{eq:defW}). Consider the eigendecomposition of ${\bf C}_j = \bU_j^T \bSigma_j \bU_j$,  where $\bSigma_j = \mbox{diag}(\varsigma_1^j,...,\varsigma_{n_p}^j)$ contains its eigenvalues,  in non-increasing order. Then,
		\begin{equation}
		{\bf C}_j \Bigl(  {\bf C}_j + \sigma^2 \; \Imat \Bigr)^{-1} =  \bU_j^T \, \bSigma_j \left( \bSigma_j + \sigma^2 \; \Imat \right)^{-1} \bU_j,
		\end{equation}
		where $\bSigma_j \left( \bSigma_j + \sigma^2\; \Imat \right)^{-1}$ is a diagonal matrix, and thus the eigenvalues  of ${\bf C}_j (  {\bf C}_j + \sigma^2\; \Imat )^{-1}$ are in 
		\[
		\bigl[ \varsigma_{n_p}^j / (\varsigma_{n_p}^j + \sigma^2),\;  \varsigma_1^j/ (\varsigma_1^j + \sigma^2) \bigr] \in [0,1),  
		\]
		since $\varsigma_{n_p}^j \geq 0$ ($\Cmat_j$ is positive semi-definite) and $\sigma^2 > 0$.  From \eqref{eq:linMMSE}, each $\bF_i$ is a convex combination of matrices, each of which with eigenvalues in $[0,1)$. Weyl's inequality \cite{Bhatia} implies that the eigenvalues of a convex combination of symmetric matrices is bounded below (above) by the same convex combination of the smallest (largest) eigenvalues of those matrices. The eigenvalues of $\bF_i$ are thus all in $[0,1)$, \textit{i.e.}, $\bF_i$ is positive semi-definite and $\|\bF_i\|_2 < 1$.
		
		Now, as in \cite{Sulam},  partition the set of patches into a collection of subsets of non-overlapping patches: $\{\Omega_j \subset \{1,..,N\}, \; j=1,...,n_p \}$ (the number of subsets of non-overlapping patches equals the patch size, due to unit stride and periodic boundaries). Using this partition, $\Wmat$ can be written as
		\begin{align}
		\Wmat =  \frac{1}{n_p} \sum_{j=1}^{n_p} \underbrace{\sum_{k \in \Omega_j} \bP_k^T \bF_k \bP_k}_{\Amat_j} = \frac{1}{n_p} \sum_{j=1}^{n_p} \Amat_j. \label{eq:noover}
		\end{align}
		Since the patches in $\Omega_j$ are disjoint, there is a permutation of the image pixels that allows writing $\Amat_j$ as a block-diagonal matrix, where the blocks are the $\bF_k$ matrices, with $k \in \Omega_j$. Because the set of eigenvalues of a block-diagonal matrix is the union of the sets of eigenvalues of its blocks, the eigenvalues of each $\Amat_j$ are bounded similarly as those of the $\bF_k$, thus in $[0,1)$. Finally, using Weyl's inequality again, the eigenvalues of $\Wmat$ are bounded above (below) by the average of the largest (smallest) eigenvalues of the $\Amat_j$, thus  in $[0,1)$. \end{proof}

		\ifCLASSOPTIONcaptionsoff
		\newpage
		\fi

		\bibliographystyle{IEEEtran}
		
		\bibliography{refs}

\end{document}